\documentclass{article}

\usepackage[preprint]{colm2025_conference}

\usepackage{microtype}
\usepackage{hyperref}
\usepackage{url}
\usepackage{booktabs}
\usepackage{multirow}
\usepackage{listings}

\usepackage[ruled, linesnumbered]{algorithm2e}
\usepackage{amsmath}

\usepackage{amsmath}
\usepackage{amsthm}
\usepackage{amssymb}
\usepackage{amsfonts}
\usepackage{dsfont}
\usepackage{graphicx}
\usepackage{adjustbox}
\usepackage{caption}
\usepackage{subcaption}

\newtheorem*{theorem*}{Theorem}

\usepackage{booktabs}
\usepackage{multirow}
\usepackage{siunitx}
\usepackage{pgfplots}
\pgfplotsset{compat=1.18}

\usepackage{lineno}

\definecolor{darkblue}{rgb}{0, 0, 0.5}
\hypersetup{colorlinks=true, citecolor=darkblue, linkcolor=darkblue, urlcolor=darkblue}

\theoremstyle{plain}
\newtheorem{theorem}{Theorem}

\newtheorem{proposition}[theorem]{Proposition}
\newtheorem*{proposition*}{Proposition}

\theoremstyle{definition}

\newtheorem*{definition*}{Definition}

\theoremstyle{remark}

\title{Economic Evaluation of LLMs}

\author{Michael J. Zellinger \\ \textbf{Matt Thomson}\\
California Institute of Technology\\
\texttt{\{zellinger, mthomson\}@caltech.edu} \\
}

\date{}

\RequirePackage{natbib}
\begin{document}

\ifcolmsubmission
\linenumbers
\fi

\maketitle

\begin{abstract}

Practitioners often navigate LLM performance trade-offs by plotting Pareto frontiers of optimal accuracy-cost trade-offs. However, this approach offers no way to compare between LLMs with distinct strengths and weaknesses: for example, a cheap, error-prone model vs a pricey but accurate one.

To address this gap, we propose \textit{economic evaluation} of LLMs. Our framework quantifies an LLM’s performance trade-off as a single number based on the economic constraints of a concrete use case, all expressed in dollars: the cost of making a mistake, the cost of incremental latency, and the cost of abstaining from a query.

We apply our economic evaluation framework to compare the performance of reasoning and non-reasoning models on difficult questions from the MATH benchmark, discovering that reasoning models offer better accuracy-cost tradeoffs as soon as the economic cost of a mistake exceeds \$0.01. In addition, we find that single large LLMs often outperform cascades when the cost of making a mistake is as low as \$0.1.

Overall, our findings suggest that when automating meaningful human tasks with AI models, practitioners should typically use the most powerful available model, rather than attempt to minimize AI deployment costs, since deployment costs are likely dwarfed by the economic impact of AI errors.
\end{abstract}

\maketitle

\section{Introduction}

Large language models (LLM) are commonly evaluated based on their accuracy, cost, latency, and other metrics (\citealp{liang2023}). Practitioners commonly display available models on an accuracy-cost scatter plot to identify models offering the best accuracy-cost trade-offs (\citealp{hu2024}). These optimal trade-offs are referred to as a ``Pareto frontier'' (\citealp{jin2006}; \citealp{branke2008}). Unfortunately, Pareto frontiers do not provide a way to rank models with distinct strengths and weaknesses. For example, it is not possible to compare a cheap, error-prone model against a pricey but accurate one. However, practitioners often face such dilemmas (\citealp{hammond2024}).

\begin{figure}[htbp]
    \centering

    \includegraphics[width=0.55\textwidth]{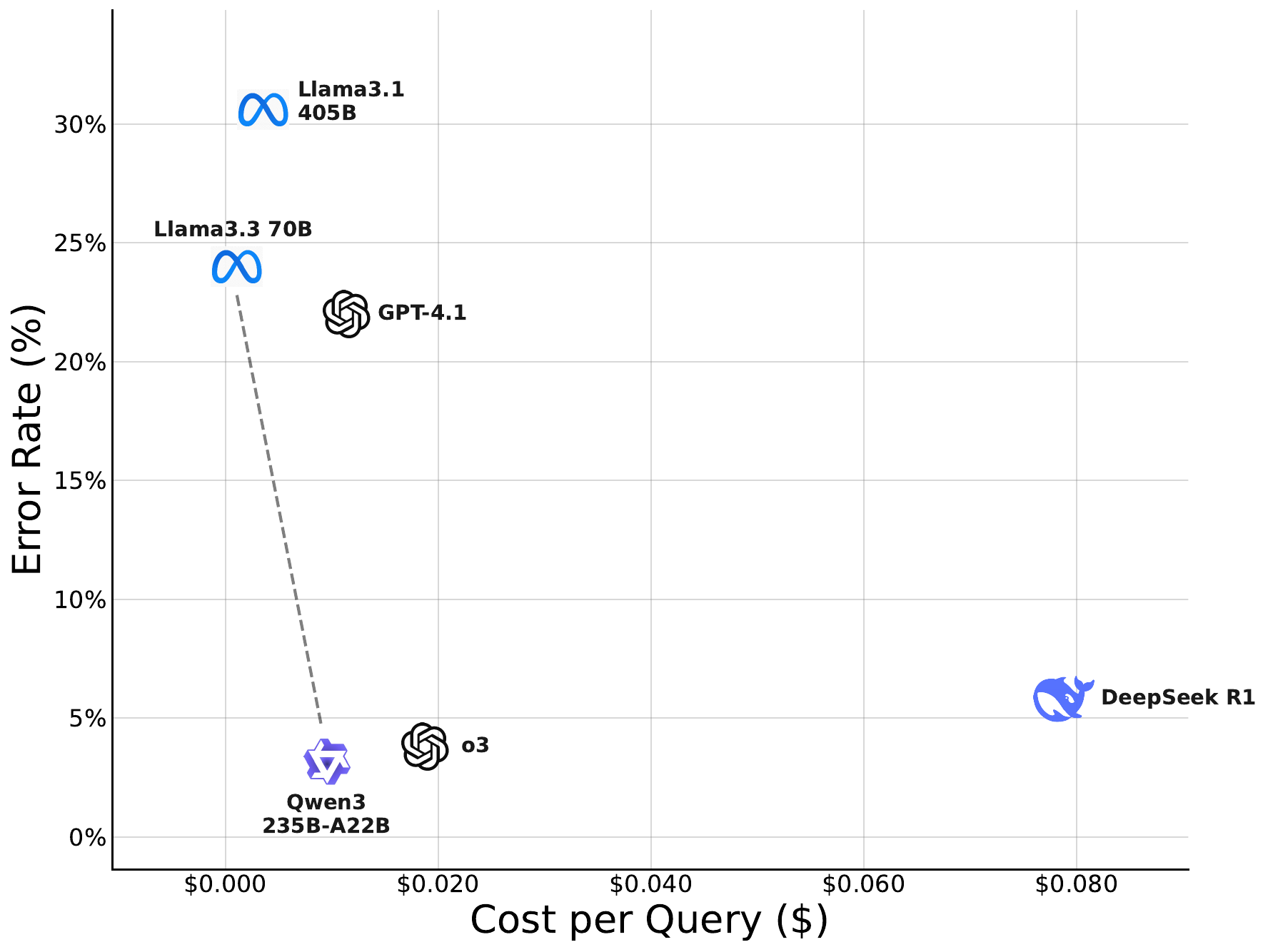}
    \hfill
    \raisebox{0.45cm}{\includegraphics[width=0.4\textwidth]{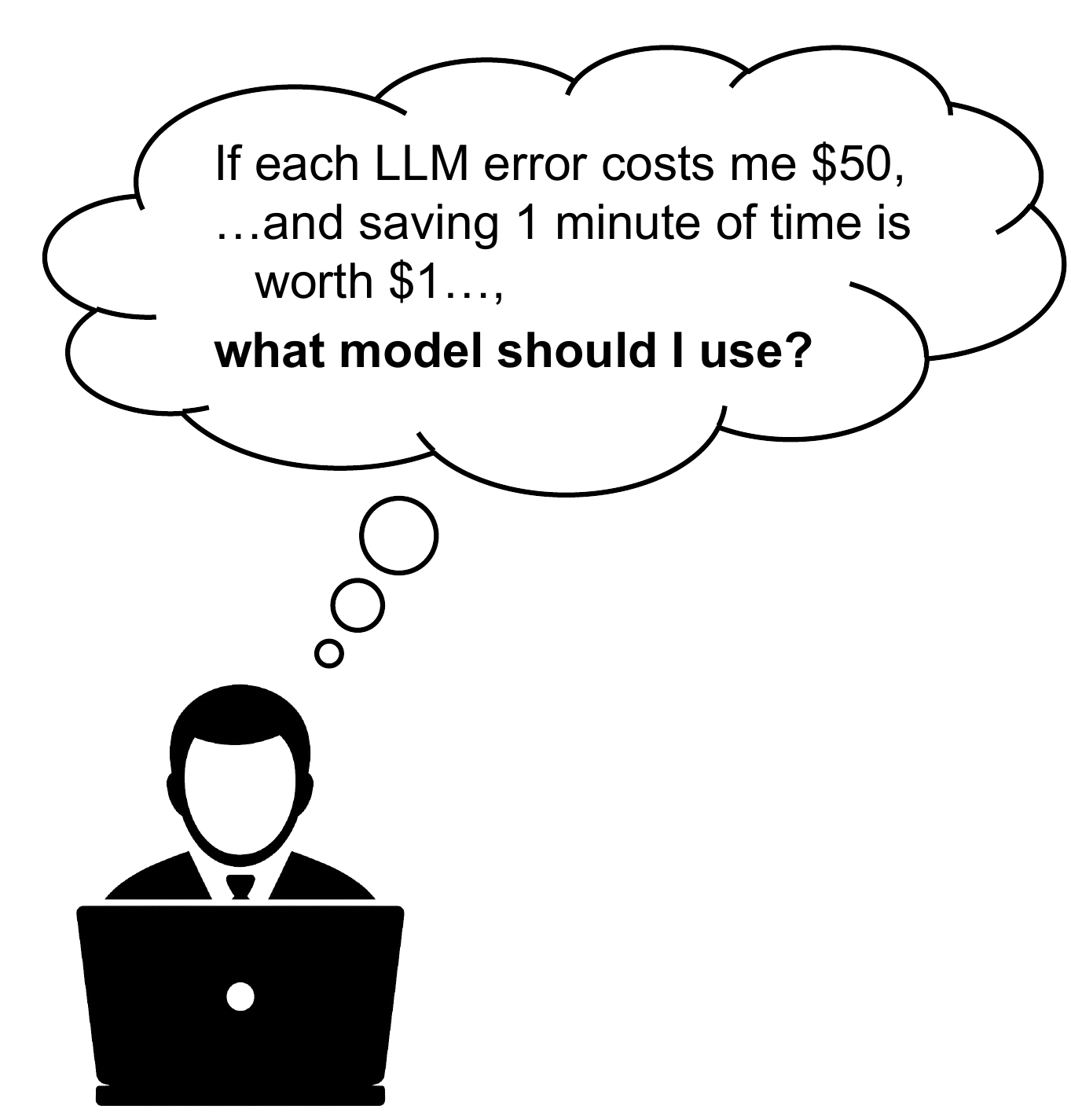}}

    \caption{Pareto frontiers of LLM performance do not reveal which model is best-suited for a given use case—a problem often faced by practitioners.}
    \label{fig:overview_figure}
\end{figure}

To address this problem, we propose an economic framework for evaluating AI models, which enables practitioners to identify the single best model for their use case, even when balancing competing objectives such as accuracy, cost, and latency. Drawing on the well-understood interpretation of Lagrange multipliers as shadow prices (\citealp{bertsekas1999}), we model a concrete use cases in terms of economic constraints expressed in dollars: the cost of making a mistake (the \textit{price of error}), the cost of incremental latency (the \textit{price of latency}), and the cost of abstaining from a query (the \textit{price of abstention}).

Given these economic parameters, our framework determines the optimal LLM. As an example, suppose a hospital deploys LLMs for medical note-taking. For this use case, the \textit{price of error} likely exceeds $\$100$, but the \textit{price of latency} may be low, perhaps $\$1$ per minute of incremental latency (equivalent to human wages of $\$60$/hour). By contrast, a natural language search engine for an e-commerce platform faces strikingly different constraints. In this domain, the price of error is much lower (perhaps $\$1$ per error) but each additional 100 milliseconds of latency are costly, since unresponsive websites drive away consumers (\citealp{kohavi2007}).

We demonstrate the practical utility of economic evaluation by addressing two important questions for practical LLM deployments. First, we ask whether reasoning or non-reasoning LLMs are optimal on complex problem-solving tasks when taking into account the reasoning models' higher dollar cost and latency. To this end, we evaluate six state-of-the-art reasoning and non-reasoning LLMs on difficult questions from the MATH benchmark (\citealp{hendrycks2021}). Our analysis shows that when latency is of no concern, reasoning models already outperform non-reasoning models when the cost of making a mistake is as low $\lambda_E^{\text{critical}} = \$0.01$. When setting the price of latency at \$10/minute (equivalent to human wages of \$600/hour), the critical price of error rises to $\lambda_E^{\text{critical}} = \$10$.

Second, we ask whether LLM cascades (\citealp{chen2023}) offer practical benefits, by comparing the performance of a cascade $\mathcal{M}_\text{small} \rightarrow \mathcal{M}_\text{big}$ against the performance of $\mathcal{M}_\text{big}$ by itself—taking into account accuracy, dollar cost, and latency. Surprisingly, we find that $\mathcal{M}_\text{big}$ typically outperforms $\mathcal{M}_\text{small} \rightarrow \mathcal{M}_\text{big}$ for prices of error as low as $\$0.1$. However, cascade performance notably depends on how well we can quantify the uncertainty of $\mathcal{M}_\text{small}$. Interestingly, using Llama3.1 405B as $\mathcal{M}_\text{small}$ yields a superior cascade that outperforms $\mathcal{M}_\text{big}$ for prices of error up to $\$10,000$, despite the fact that Llama3.1 405B performs comparatively poorly as a standalone model.

To supplement our experiments, we furnish theoretical results on 1) explaining the performance of a cascade $\mathcal{M}_\text{small} \rightarrow \mathcal{M}_\text{big}$ in terms of a novel covariance-based metric measuring the quality of $\mathcal{M}_\text{small}$'s uncertainty signal, and 2) connecting our economic evaluation framework to standard multi-objective optimization based on Pareto optimality.

In summary, our key contributions are the following:
\begin{itemize}
    \item We propose an economic framework for evaluating LLMs (and LLM systems), which determines a single optimal model based on a use case's economic constraints, all expressed in dollars: the cost of making a mistake (\textit{price of error}), the cost of incremental latency (\textit{price of latency}), and the cost of abstaining from a query (\textit{price of abstention}).
    \item We state empirical dollar figures for the critical cost of error at which reasoning LLMs offer superior accuracy-cost-latency trade-offs compared to non-reasoning models; and at which a single large LLM $\mathcal{M}_\text{big}$ outperforms a cascade $\mathcal{M}_\text{small} \rightarrow \mathcal{M}_\text{big}$.
\end{itemize}

Overall, our findings suggest that when automating meaningful human tasks using AI models, accuracy is likely the most important economic factor, outweighing inference costs. Hence, practitioners should typically deploy the most powerful available LLMs.

\section{Background}

\noindent \textbf{Large language models}. Large language models (LLM) are transformer-based deep neural networks that auto-regressively generate ``tokens'' of text, one token at a time (\citealp{vaswani2017}). A model $\mathcal{M}$ parametrizes a probability distribution $p_\theta(x_{t+1} | x_1, ..., x_{t})$ over the next token given all previous tokens. To generate new text, the user provides a \textit{prompt} $(x_1, ..., x_{t})$ and then auto-regressively samples new tokens $x_{t+1}, x_{t+2}, ...$ according to the conditional distribution $p_\theta$. After \textit{post-training} a model's parameters $\theta$ using supervised fine-tuning (\citealp{wei2022}; \citealp{chung2022}) and reinforcement learning (\citealp{ouyang2022}), sampling from $p_\theta$ returns an appropriate response to the prompt. For example, if the prompt $(x_1, ..., x_{t})$ encodes the English sentence ``I want to finish all the research articles I start reading. What can I do to improve my self-discipline?'', auto-regressive sampling from the post-trained $p_\theta$ should yield a list of actionable suggestions.\footnote{You can start today, with this one!}

\noindent \textbf{Systems of large language models}. The development of large language models has proven out theoretical scaling laws (\citealp{kaplan2020}) predicting that better performance reliably follows from simply implementing bigger models (\citealp{brown2020}; \citealp{chowdhery2022}; \citealp{openai2024}). As a result, parameter counts for state-of-the-art LLMs have surged to more than one trillion, requiring significant computing resources for both training and inference. Noting that smaller LLMs often perform well on easier tasks, researchers have proposed system of LLMs in which small and large models collaborate to enhance computational efficiency and reduce inference costs (\citealp{chen2023}). For example, \textit{cascades} (\citealp{ding2024}; \citealp{wang2024}; \citealp{narasimhan2024}) delegate queries from small to large LLMs only if the small LLMs are uncertain about the answer, and \textit{routers} (\citealp{hari2023}; \citealp{hu2024}; \citealp{ong2025}) directly send each query to the smallest available model that can still return a satisfactory answer.

\noindent \textbf{Pareto Optimality}. Multi-objective optimization (\citealp{coello2005}; \citealp{jin2006}) is concerned with minimizing a vector-valued function
\begin{equation}
\label{eq:multiobjective}
    \theta^{*} = \arg \min~ F(\theta),
\end{equation}
where $F(\theta) = (f_1(\theta), ..., f_k(\theta)) \in \mathbb{R}^{k}$. Since the distinct objectives $f_i(\theta)$ generally conflict, we define solutions to (\ref{eq:multiobjective}) with respect to Pareto optimality (\citealp{branke2008}). We say that $(f_1(\theta_1), ..., f_k(\theta_1)$ \textit{Pareto-dominates}  $(f_1(\theta_2), ..., f_k(\theta_2)$, or 
\begin{equation}
    (f_1(\theta_1), ..., f_k(\theta_1) <_\text{P} (f_1(\theta_2), ..., f_k(\theta_2),
\end{equation}
if $f_i(\theta_1) \leq f_i(\theta_2)$ for all $i=1, 2, ..., k$ and at least one of the inequalities is strict. We consider $\theta^{*}$ to be a solution of (\ref{eq:multiobjective}) if there exists no $\theta'$ such that $(f_1(\theta'), ..., f_k(\theta') <_\text{P} (f_1(\theta^{*}), ..., f_k(\theta^{*})$. The set of such solutions to $(\ref{eq:multiobjective})$ makes up the \textit{Pareto frontier}
\begin{equation}
    \mathcal{P} = \{ \theta \in \Theta \mid \forall \theta' \in \Theta, ~F(\theta') \not<_P F(\theta) \}.
\end{equation}
We interchangeably refer to the image $F(\mathcal{P})  = \{ (f_1(\theta), ..., f_k(\theta)) \mid \theta \in \mathcal{P} \} $ as the Pareto frontier. The multi-objective performance of LLMs and LLM systems is typically evaluated by computing the Pareto frontier of the performance metrics $f_1(\theta) := \text{error rate}$ and $f_2(\theta) := \text{cost}$.

When comparing the performance of individual LLMs, $\theta$ represents the model's identity (for example, GPT 4.1 or DeepSeek R1) and other hyperparameters (user and system prompts, sampling configuration, etc.). By contrast, for systems of LLMs, $\theta$ typically denotes the system's operating point, e.g., the deferral rate or confidence threshold for an LLM cascade $\mathcal{M}_\text{small} \rightarrow \mathcal{M}_\text{big}$.

\section{Economic Evaluation of LLMs}

In this section, we describe our economic framework for evaluating the multi-objective performance trade-offs of LLMs (and systems of LLMs).

\subsection{LLMs as Agents}
\label{subsec:llms-as-agents}

Using the language of reinforcement learning, we cast LLMs and systems of LLMs as \textit{agents} that reap per-query \textit{rewards} from their chosen \textit{actions}. As we will show, this methodology provides a natural basis for multi-objective optimization of LLM performance from an economic perspective. Rather than present our framework in the abstract, we illustrate our formalism for three concrete LLM systems: standalone LLMs, cascades (\citealp{chen2023}), and routers (\citealp{shnitzer2023}).

\noindent \textbf{Standalone LLM}. An LLM's action space consists of its possible text generations:
\begin{equation}
    \mathcal{A} = \{ y \mid y \in \Sigma^{*} \},
\end{equation}

where $\Sigma$ is the alphabet of tokens (\citealp{kudo2018}). Following other authors, to evaluate the quality of the output $y$ we assume the existence of a binary error-calling mechanism $s(y): \Sigma^{*} \to \{0,1\}$ that maps each string output to a judgment of whether the output is ``satisfactory.''\footnote{Optionally, the error-calling function $s$ may take into account a reference answer $y_\text{ref}$, leading to a bivariate function $s(y, y_\text{ref}): \Sigma^{*} \times \Sigma^{*} \to \{ 0,1 \}$.} We refer to each $y$ with $s(y) = 1$ as an \textit{error} and use the notation $\mathds{1}_E(y)$ as a shorthand for the indicator $\mathds{1}[s(y) = 1]$.

For each action $y \in \mathcal{A}$, the LLM reaps a reward
\begin{equation}
    r = - \left ( C + \lambda_L L + \lambda_E \mathds{1}_E \right ),
\end{equation}
where $C$ is the dollar cost of generating the output $y$, and $L$ is the latency.

\noindent \textbf{LLM Cascade}. A \textit{cascade} $\mathcal{C} = M_1 \rightarrow M_2 \rightarrow ... \rightarrow M_k$ is a system of LLMs $M_1$, $M_2$, ..., $M_k$ that passes each incoming query from $M_{i}$ to $M_{i+1}$ until it encounters a model $M_\tau$ ($1 \leq \tau \leq k$) with sufficient confidence to answer the query.

We think of the entire cascade as an agent with action space
\begin{equation}
    \mathcal{A} = \{ (\tau, y) \mid 1 \leq \tau \leq k, y \in \Sigma^{*} \},
\end{equation}
where $y$ denotes the cascade's output, and $\tau$ represents the index of the model $M_\tau$ responsible for the cascade's output $y$. Specifically, $\tau$ is the index of the first model that does not defer the query.

As for a standalone LLM, the cascade reaps a reward $r = - \left ( C + \lambda_L L + \lambda_E \mathds{1}_E \right )$ for each action $(\tau, y)$. However, the cost and latency depend on $\tau$:
\begin{equation}
    C = \sum_{j=1}^{\tau} C_j, \qquad L = \sum_{j=1}^{\tau} L_j,
\end{equation}
where $C_j$ and $L_j$ are the cost and latency of model $M_j$ on the query. These equations result from the ``cascading'' nature of a cascade: we pay for each successive model—both in dollar cost and in latency—until we reach the earliest model with sufficient confidence, $M_\tau$.

\noindent \textbf{LLM Router with Abstention}. A \textit{router with abstention} is a function $g(x): \Sigma^{*} \to \{1, 2, ..., k\} \cup \{ \varnothing \} $ that routes each incoming query $x$ to one of $k$ LLMs $M_1$, ..., $M_k$ in a single step, or abstains from answering the query ($\varnothing$). Its action space is
\begin{equation}
    \mathcal{A} = \{ (i, y) \mid i \in \{1, 2, ..., k\}, y \in \Sigma^{*} \} \cup \{ \varnothing \},
\end{equation}
where $(i, y)$ denotes that model $i$ generates output $y$, and $\varnothing$ indicates that the router abstained (for example, to defer the query to a human expert). Analogous to a cascade, the router with abstention reaps a reward $r = - \left[ C + \lambda_L L + \lambda_E \mathds{1}_E(y) + \lambda_A \mathds{1}_A(y) \right]$ for each action, the sole difference being the addition of a term $\mathds{1}_A$ indicating that the router abstained. However, the cost $C$ and latency $L$ are computed differently:
\begin{equation}
    C = c_0 + C_i, \qquad L = l_0 + L_i,
\end{equation}
where $C_i$ and $L_i$ are the cost and latency of generating the output with model $M_i$. The constants $c_0$ and $l_0$ represent the computational overhead of determining the routing decision $g(x) \in \{ 1, 2, ..., k\}$. These costs are typically negligible, as $g$ is usually lightweight compared to the LLMs $M_1$, $M_2$, ..., $M_k$.\footnote{Typically, the routing model takes the form of a deep neural network—for example, a finetuned small language model—with less than 1B parameters (\citealp{shnitzer2023}; \citealp{hari2023}).}

\subsection{Economic Modeling of LLM Use Cases}
\label{subsec:quantifying_user_preferences}

To cast standalone LLMs, cascades, and routers as reward-maximizing agents, we defined the per-query reward
\begin{equation}
\label{eq:per-query-reward-concrete}
    r = - \left[ C + \lambda_L L + \lambda_E \mathds{1}_E + \lambda_A \mathds{1}_A \right],
\end{equation}
where $C$ is the dollar cost of processing a query, $L$ is the latency, $\mathds{1}_E$ indicates an error, and $\mathds{1}_A$ indicates an abstention. In general, we formulate this reward as 
\begin{equation}
\label{eq:per-query-reward-abstract}
    r = - \left[ C + \sum_{\mu \in \mathcal{P}_\text{numeric}} \lambda_\mu ~\mu ~+ \sum_{\chi \in \mathcal{P}_\text{binary}} \lambda_\chi \mathds{1}_{\chi} \right],
\end{equation}
where $\mathcal{P}_\text{numeric} \cup \mathcal{P}_\text{binary}$ is the set of per-query performance metrics: $\mathcal{P}_\textit{numeric}$ represents the numeric performance metrics (for example, latency) and $\mathcal{P}_\textit{binary}$ denotes the binary performance events (for example, error and abstention). This definition includes other reasonable performance objectives, such as privacy (\citealp{zhang2025}).

The coefficients $\{ \lambda_\mu \}_{\mu \in \mathcal{P}_\text{numeric}}, \{ \lambda_\chi \}_{\chi \in \mathcal{P}_\text{binary}} \in \mathbb{R}^{+}$ are \textit{prices} measuring the economic impact when the performance metrics worsen. Table \ref{tbl:lambda_explanation} gives a few key examples:

\begin{table}[ht]
  \centering
  \caption{Examples of key economic parameters in our framework.}
  \label{tbl:lambda_explanation}
  \begin{tabular}{@{}l c p{7.5cm} c@{}}
    \toprule
    \textbf{Parameter} & \textbf{Symbol} & \textbf{Definition} & \textbf{Units} \\
    \midrule
    Price of Error      & $\lambda_E$ & Amount the user is willing to pay to avoid a single prediction error. & \$ \\
    Price of Latency    & $\lambda_L$ & Amount the user is willing to pay to reduce per‑query latency by one second. & \$/sec \\
    Price of Abstention & $\lambda_A$ & Amount the user is willing to pay to avoid a model abstention (no answer). & \$ \\
    \bottomrule
  \end{tabular}
\end{table}

These parameters are based on the economic concept of \textit{indifference} (\citealp{mankiw2020}). For example, the price of error $\lambda_E$ is the lowest dollar figure $d$ at which the user (i.e., the organization that deploys the LLM system) would be indifferent between suffering an LLM error or receiving $d$ dollars in cash.

\subsection{Multi-Objective Performance Evaluation in a Single Number}

Given the user's economic constraints (see Table \ref{tbl:lambda_explanation}), the expected per-query reward of the LLM system is
\begin{equation}
    R(\lambda; \theta) = \mathbb{E}_{\theta} [r(\lambda)],
\end{equation}
where $r$ is the per-query reward (\ref{eq:per-query-reward-abstract}). We use $\lambda$ to denote the totality of economic parameters (e.g., $\lambda_E$, $\lambda_L$, ...), and denote the configuration of the LLM system by $\theta$. Selecting the optimal LLM for a given use case, or optimizing the operating point of a system of LLMs, involves the reward maximization
\begin{equation}
    \label{eq:optimal_theta}
    \theta^{*}(\lambda) = \arg\max_\theta R(\lambda ; \theta).
\end{equation}
When choosing among individual LLMs, $\theta$ represents the identity of the model (e.g., GPT 4.1 vs DeepSeek R1), as well as hyperparameter settings (user and system prompts, sampling temperature, top-p, top-k, etc.). On the other hand, for LLM systems, $\theta$ typically represents tunable parameters such as the confidence thresholds of LLM cascades (\citealp{zellinger2025}).

Often the user's economic constraints $\lambda$ are not known with certainty. In this case, it is instructive to compute optimal models for a range of potential $\lambda$ values. These sensitivity tables can be highly informative, as the optimal model is often stable over a wide range of different economic constraints. See sections \ref{subsec:reasoning_q} and \ref{subsec:reasoning_q} for examples. 

To compare the performance of different LLMs (or systems of LLMs) across different economic scenarios $\lambda$, we consider their expected per-query rewards for the optimal choices of $\theta$:
\begin{equation}
    \label{eq:expected_reward}
    R(\lambda) = R(\lambda; \theta^{*}(\lambda)).
\end{equation}

Alternatively, if $\lambda$ is uncertain—for example, suppose that new legislation may change the expected payout of medical malpractice lawsuits, potentially raising the price of error $\lambda_E$ for medical AI deployments—we model $\lambda$ as a random variable $\lambda \sim p(\lambda)$, yielding the expected per-query reward
\begin{equation}
    \label{eq:stochastic-per-query-reward}
    R = \mathbb{E}_{\lambda \sim p(\lambda)} [R(\lambda)].
\end{equation}

\subsection{Estimating the Price of Error}
\label{subsec:poe_estimate}

Managing the potential cost of LLM mistakes is critical for businesses, especially those in risk-sensitive industries (e.g., finance, law, or medicine).

In this section, we illustrate how to estimate the price of error by walking through an example calculation for medical diagnosis. Practitioners may then adapt these steps to their own industries.

\noindent \textbf{Estimate of $\boldsymbol{\lambda}_\mathbf{E}$ for medical diagnosis.} We estimate the price of error for medical diagnosis to be about \$1,000. We arrive at this number by considering data on medical malpractice lawsuits and applying Bayes' theorem.

For a single diagnosis, denote the event of a medical malpractice lawsuit as $M$ and the event of a diagnostic error as $E$. Our estimate for the price of error is then
\begin{align}
    \hat{\lambda}_E & = \mathbb{E}[\text{Cost} | M] \times \mathbb{P}(M | E) \\
    & = \mathbb{E}[\text{Cost} | M] \times \frac{\mathbb{P}(E|M) \mathbb{P}(M)}{\mathbb{P}(E)}. \label{poe_estimate}
\end{align}

\citet{studdert2006} report that the mean payout for medical malpractice lawsuits is \$485,348, so we use $\mathbb{E}[\text{Cost} | M] = \$500,000$. In addition, two thirds of malpractice suits are derived from a genuine medical error (rather than a fraudulent claim), so $\mathbb{P}(E|M) = 2/3$.

\citet{anupam2011} estimate a doctor's yearly risk of facing a malpractice claim as 7.4\%, so we estimate that a doctor encounters a malpractice suit once every 1/0.074 = 13.5 years. Assuming the doctor makes 3 diagnoses per hour, we arrive at around 100,000 diagnoses within this time frame, so $\mathbb{P}(M) \approx 1/100,000$. By contrast, \citet{singh2014} estimate that 1 in 20 adults experiences a diagnostic error each year. Taking into account the fact that people may go to the hospital more than once per year, and that each visit may involve more than one diagnosis, we use $\mathbb{P}(E) \approx 1/100$.

Plugging these numbers into the formula (\ref{poe_estimate}), we arrive at the estimate \boxed{\hat{\lambda}_E \approx \$333} for medical diagnosis.

\subsection{Connections between Economic Evaluation and Pareto Optimality}

In this section, we establish theoretical connections between our economic evaluation framework and Pareto optimality. Our first result shows that sweeping over different economic scenarios $\lambda$ recovers the full Pareto frontier for the performance metrics, assuming regularity conditions.

\begin{theorem}
\label{thm:pareto_mapping}
Let $\theta^{*}(\lambda)$ be the solution to the reward maximization problem
\begin{equation}
\label{eq:reward_opt}
\begin{aligned}
    \theta^{*} =~ & \text{argmax}_\theta && R(\lambda; \theta),
\end{aligned}
\end{equation}
where $\theta \in \mathbb{R}^{p}$ denotes an LLM system's tunable parameters, and $\lambda$ is the vector of economic costs as defined in Section \ref{subsec:quantifying_user_preferences}. Assume that regularity conditions hold, such that for each $\lambda \in \mathbb{R}_{>0}^{|\mathcal{P}_\text{numeric}| + |\mathcal{P}_\text{binary}|}$ there exist bounds $\{ \gamma_\mu \}_{\mu \in \mathcal{P}_\text{numeric}}$ and $\{ \gamma_\chi \}_{\mu \in \mathcal{P}_\text{binary}} > 0$ such that $\theta^{*}(\lambda)$ is equivalently the solution of the constrained optimization problem
\begin{equation}
\label{eq:constrained_opt_general}
\begin{aligned}
    \theta^{*} =~ & \text{argmin}_\theta && \hat{\mathbb{E}}_\theta[C] \\
    & \text{subject to} && \hat{\mathbb{E}}_\theta[\mu] \leq \gamma_\mu, ~~\mu \in \mathcal{P}_\text{numeric} \\
    & && \hat{\mathbb{E}}_{\theta}[\mathds{1}_{\chi}] \leq \gamma_\chi, ~~\chi \in \mathcal{P}_\text{binary},
\end{aligned}
\end{equation}
and vice versa for $\gamma \mapsto \lambda(\gamma)$. Then the vector of economic costs, $\lambda$, maps surjectively onto the Pareto surface via the mapping
\begin{equation}
\label{eq:pareto_mapping}
    \lambda \mapsto (\hat{\mathbb{E}}_{\theta^{*}(\lambda)}[C], \hat{\mathbb{E}}_{\theta^{*}(\lambda)}[\mu_1], ..., \hat{\mathbb{E}}_{\theta^{*}(\lambda)}[\mu_{|\mathcal{P}_\text{numeric}|}], \hat{\mathbb{P}}_{\theta^{*}(\lambda)}[\chi_1], ..., \hat{\mathbb{P}}_{\theta^{*}(\lambda)}[\chi_{|\mathcal{P}_\text{binary}|}]).
\end{equation}
\end{theorem}

\begin{proof}
See Appendix B.
\end{proof}

The next result provides theoretical support for evaluating the overall performance of LLM systems by comparing the expected reward (\ref{eq:expected_reward}) across a grid of possible use cases $\lambda$.

\begin{theorem}[]
Consider two LLM systems, $\mathcal{S}_1$ and $\mathcal{S}_2$. Assume that the regularity assumptions of Theorem \ref{thm:pareto_mapping} hold. If the expected rewards (\ref{eq:expected_reward}) satisfy
\begin{equation*}
    R_1(\lambda) \geq R_2(\lambda)
\end{equation*}
for all $\lambda \in \mathbb{R}_{>0}^{|\mathcal{P}_\text{numeric}| + |\mathcal{P}_\text{binary}|}$, then no point on the Pareto surface for $\mathcal{S}_1$ dominates any point on the Pareto surface for $\mathcal{S}_2$.
\end{theorem}

\begin{proof}
Suppose for the sake of contradiction that $\theta_2$ on the Pareto surface of $\mathcal{S}_2$ dominates $\theta_1$ on the Pareto surface of $\mathcal{S}_1$. By Theorem \ref{thm:pareto_mapping}, there exist $\lambda_1$, $\lambda_2 \in \mathbb{R}^{+}$ such that
\begin{align}
    \theta_1 & = \theta^{*}(\lambda_1), \\
    \theta_2 & = \theta^{*}(\lambda_2),
\end{align}
as defined by (\ref{eq:optimal_theta}). Hence, we have
\begin{equation}
    R_1(\lambda_1) = R_1(\lambda_1; \theta^{*}(\lambda_1)) = R(\lambda_1; \theta_1) < R_2(\lambda_1; \theta_2) \leq \max_\theta R_2(\lambda_1; \theta) = R_2(\lambda_1),
\end{equation}
where the middle inequality follows from the assumed Pareto dominance of $\theta_2$ over $\theta_1$.
\end{proof}

\section{Experiments}

We apply our economic evaluation framework to explore the practical relevance of less powerful LLMs. We address two concrete questions:
\begin{itemize}
    \item When do reasoning models outperform non-reasoning models? (Section \ref{subsec:reasoning_q})
    \item When does a single large model $\mathcal{M}_\text{big}$ outperform a cascade $\mathcal{M}_\text{small} \rightarrow \mathcal{M}_\text{big}$? (Section \ref{subsec:cascade_q})
\end{itemize}

\subsection{Methodology}

Since AI models are increasingly considered as a replacement for meaningful human labor, we focus our analysis on difficult mathematics questions from the MATH benchmark (\citealp{hendrycks2021}). This benchmark offers the advantage of containing ground truth difficulty labels (1-5) as well as reference answers. To simplify evaluation of LLM answers, we filter out questions with non-numeric answers and use stratified sampling to obtain 500 questions for each of the three difficulty levels 1, 3, and 5. We exclusively use the training split of the MATH benchmark; although this choice heightens the potential for data contamination (\citealp{ravaut2025}), it makes available a greater number of difficult examples and therefore improves the statistical power of our experiments.

\noindent \textbf{Models}. We evaluate six state-of-the-art LLMs (three reasoning, three non-reasoning): Meta's Llama3.3 70B and Llama3.1 405B models (\citealp{grattafiori2024}), OpenAI's GPT4.1 and o3 models, DeepSeek's R1 model (\citealp{deepseek2025}), and Alibaba's Qwen3 235B-A22 model (\citealp{yang2025}). We prompt each model using zero-shot chain-of-thought (\citealp{kojima2022}).

\noindent \textbf{LLM cascades}. To evaluate cascades $\mathcal{M}_\text{small} \rightarrow \mathcal{M}_\text{big}$, we quantify the small model's probability of correctness using self-verification, also known as \textit{P(True)} (\citealp{kadavath2022}; \citealp{zellinger2025}). When comparing the performance of $\mathcal{M}_\text{small} \rightarrow \mathcal{M}_\text{big}$ against that of $\mathcal{M}_\text{big}$ in Section \ref{subsec:cascade_q}, we split the data into 50\% training and test sets; we use the training set exclusively for estimating the optimal confidence threshold, and evaluate cascade performance on the test split ($n=250$).

\noindent \textbf{Metrics}. We measure the correctness, dollar cost, and latency for each query by invoking LLMs via the commercial Fireworks AI (Llama3.1, Llama3.3, Qwen3, DeepSeek R1) and OpenAI (GPT 4.1, o3) application programming interfaces.

\noindent \textbf{Correctness}. To assess correctness of a model's answer, we invoke Llama3.1 405B with an evaluation prompt containing the ground truth reference answer.

\noindent \textbf{Cost}. We compute
\begin{equation}
    C = N_\text{in} \times C_\text{in} + N_\text{out} \times C_\text{out},
\end{equation}
where $N_\text{in}, N_\text{out}$ are the numbers of input and output tokens, and $C_\text{in}, C_\text{out}$ are the API providers' model-specific prices, expressed in dollars per token.

\noindent \textbf{Latency}. We record the time before and after an API call to the LLM model provider. Hence, our reported latencies include internet roundtrip latency. However, this additional latency ($<$300ms) is negligible, being 10-200x smaller than the latencies we observe for answering queries.

\noindent \textbf{Cascade Error Reduction}: To predict the performance of a cascade $\mathcal{M}_\text{small} \rightarrow \mathcal{M}_\text{big}$ based on the quality of the confidence signal of $\mathcal{M}_\text{small}$ (self-verification in our case), we introduce the \textit{cascade error reduction}
\begin{equation}
    \text{CER} = \text{Cov}(\mathds{1}_D, \mathds{1}_\text{error}^{\mathcal{M}_\text{small}}),
\end{equation}
where $\mathds{1}_D$ indicates the small model's decision to defer the query to $\mathcal{M}_\text{big}$, and $\mathds{1}_\text{error}^{\mathcal{M}_\text{small}}$ indicates that the output of $\mathcal{M}_\text{small}$ is incorrect. We theoretically justify this metric in Theorem \ref{thm:cascade_error}.

See appendices C-E for more details on methodology.

\subsection{Baseline Performance: Error Rate, Cost, and Latency}

Figure \ref{fig:baseline_performance_difficulty_5} shows the performance of reasoning and non-reasoning models on $n=500$ of the most difficult questions of the MATH benchmark. Clearly, reasoning models have much lower error rates. However, their costs per query are 10-100x greater, and latencies per query are up to 10x greater.

\begin{figure}[htbp]
    \centering

    \begin{subfigure}[b]{0.32\textwidth}
        \centering
        \includegraphics[width=\textwidth]{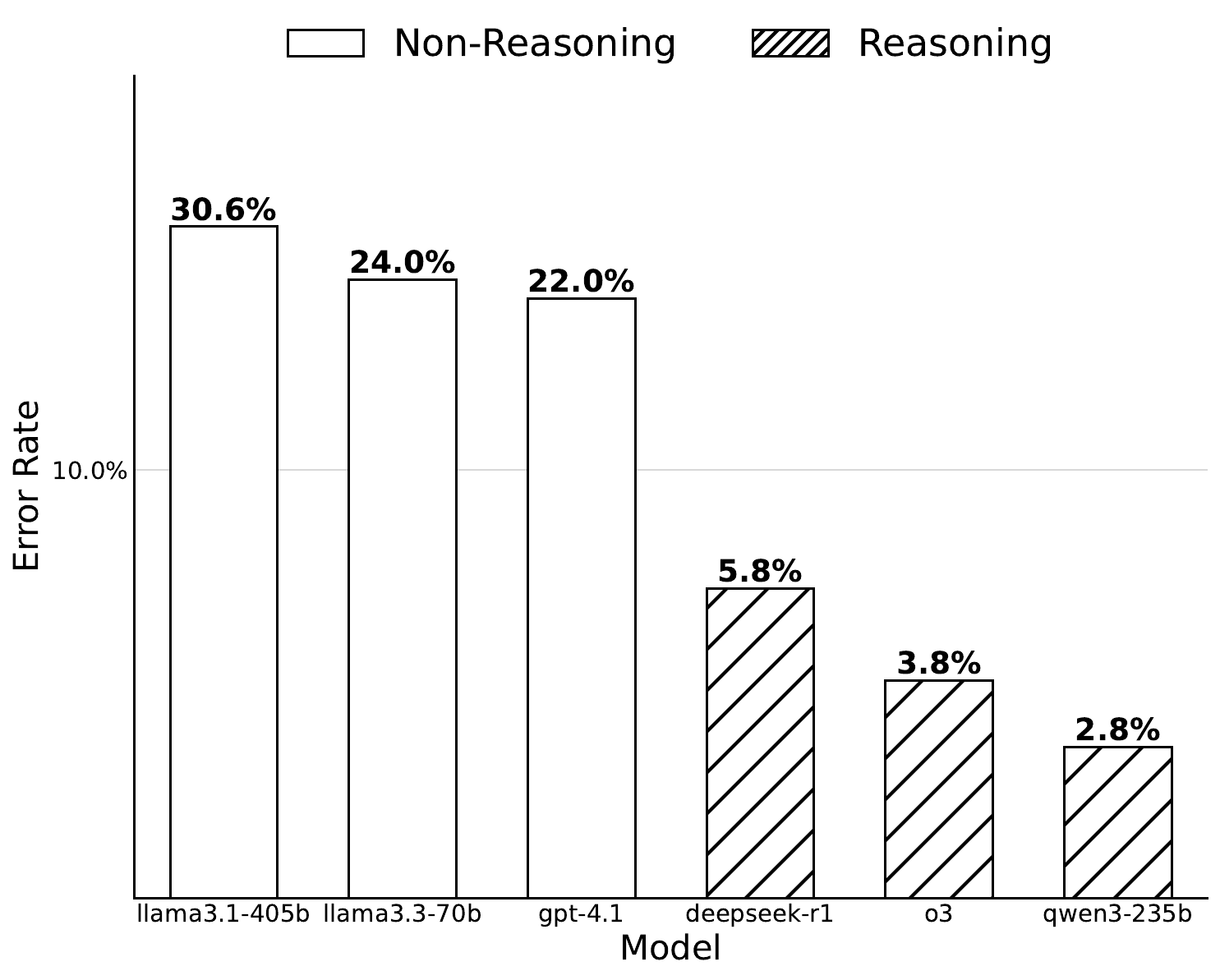}
        \caption{Error}
        \label{fig:baseline_performance_difficulty_5_error}
    \end{subfigure}
    \hfill
    \begin{subfigure}[b]{0.32\textwidth}
        \centering
        \includegraphics[width=\textwidth]{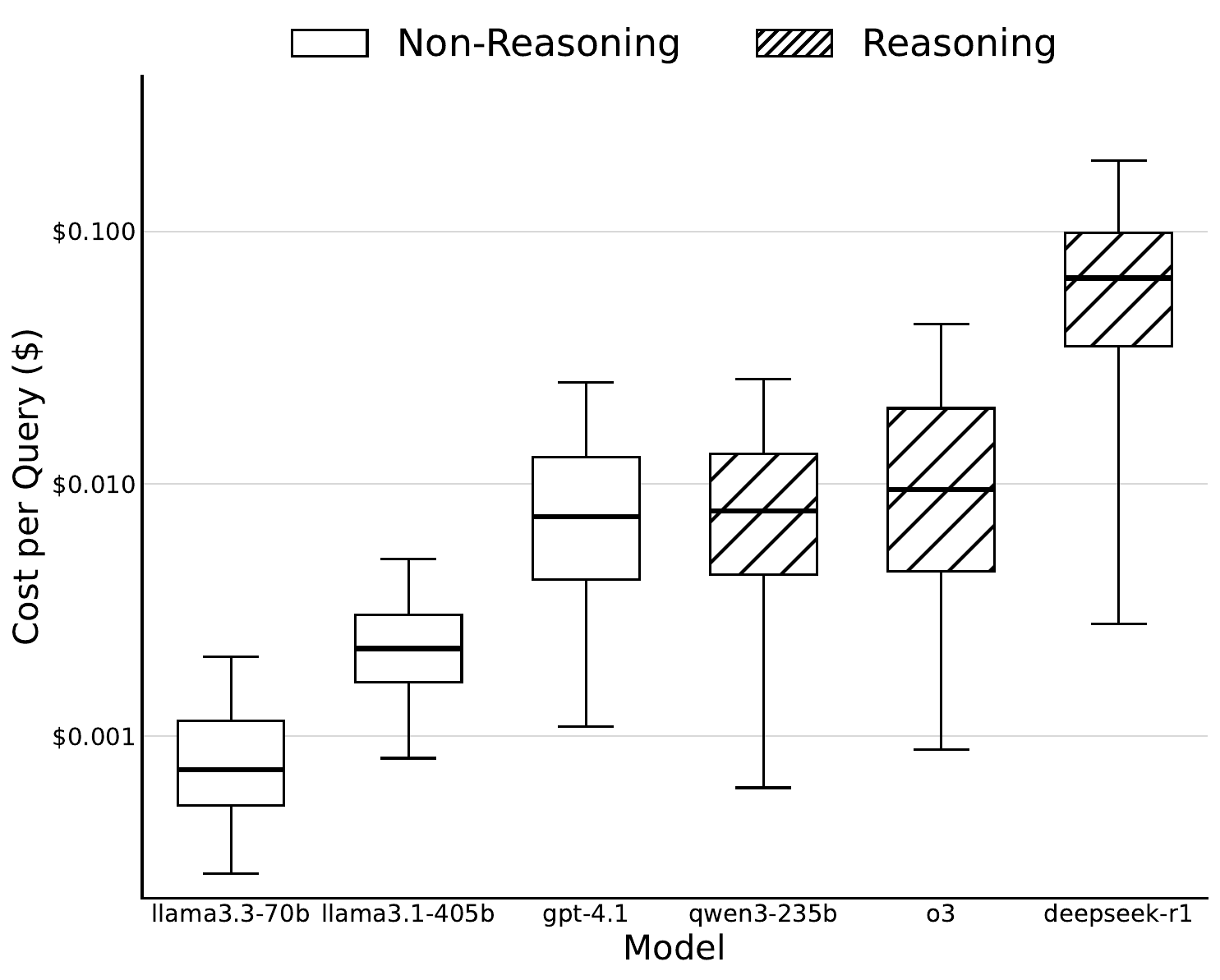}
        \caption{Cost per Query (\$)}
        \label{fig:baseline_performance_difficulty_5_cost}
    \end{subfigure}
    \hfill
    \begin{subfigure}[b]{0.32\textwidth}
        \centering
        \includegraphics[width=\textwidth]{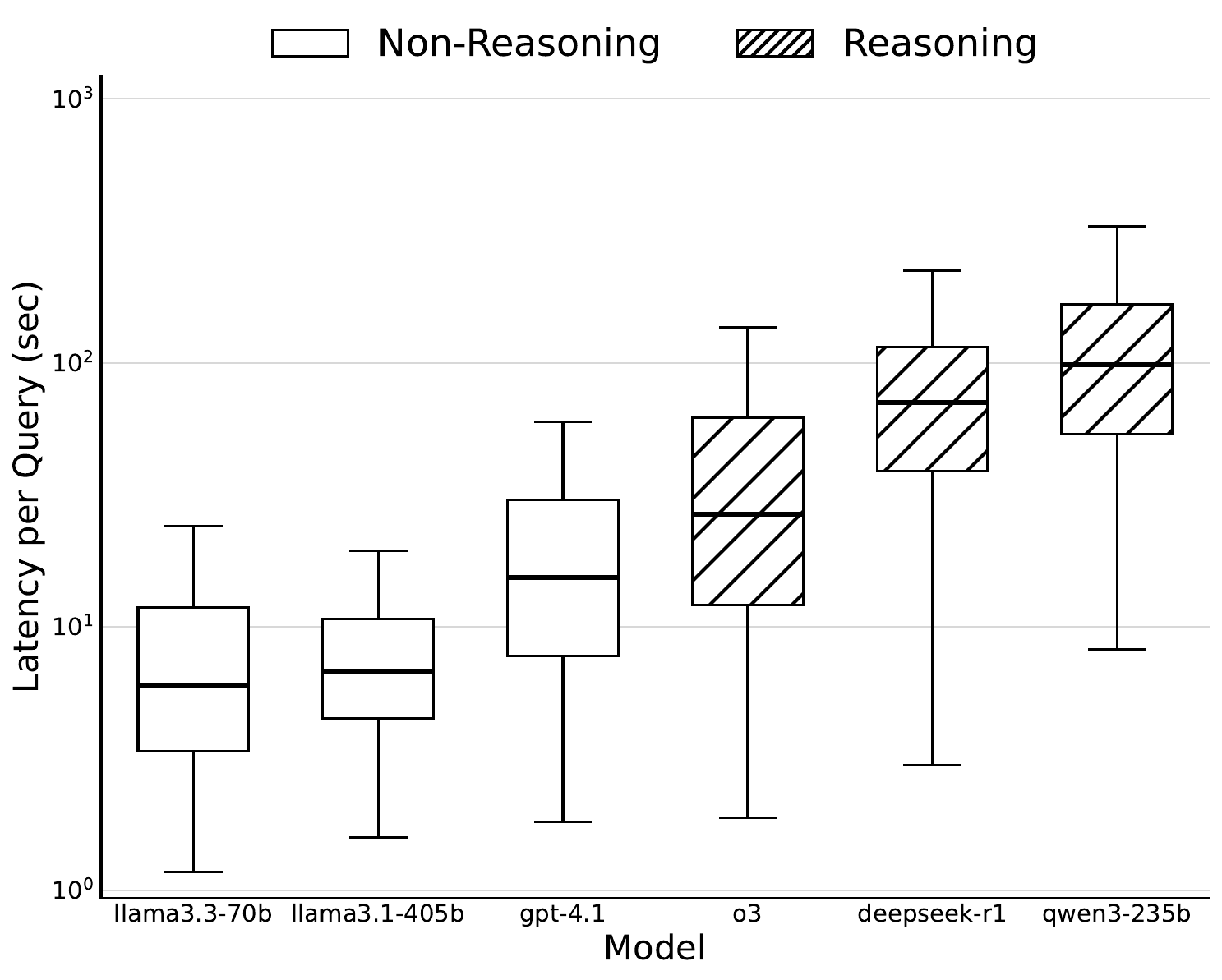}
        \caption{Latency per Query (sec)}
        \label{fig:baseline_performance_difficulty_5_latency}
    \end{subfigure}

    \caption{On the most difficult questions of the MATH benchmark, reasoning models have much lower error rates but are 10-100x more expensive and take 10x longer to answer a query.}
    \label{fig:baseline_performance_difficulty_5}
\end{figure}

Reasoning models are known to dynamically scale the number of output tokens based on the difficulty of the query. However, non-reasoning models prompted with chain-of-thought also adapt the number of output tokens based on query difficulty. Figure \ref{fig:baseline_performance_tokens} in Appendix A compares the number of output tokens across queries of varying difficulty, showing that the relative increase in output tokens is comparable between reasoning and non-reasoning models; however, reasoning models have higher baseline numbers of output tokens.

\subsection{When Do Reasoning Models Outperform Non-Reasoning Models?}
\label{subsec:reasoning_q}

We now leverage our economic evaluation framework to determine under what conditions reasoning models outperform non-reasoning models on difficult questions from MATH.

First, we only trade-off accuracy and cost, disregarding latency. Second, we simultaneously consider accuracy, cost, and latency, and display the optimal LLM over a range of economic constraints.

Figure \ref{fig:reasoning_vs_nonreasoning} plots the expected reward (\ref{eq:expected_reward}) of reasoning and non-reasoning LLMs on MATH, for prices of error ranging from $\lambda_E = \$0.0001$ to $\lambda_E = \$10,000$ per query. These curves only trade-off accuracy and cost, assuming that latency costs nothing ($\lambda_L = \$0/\text{sec}$). Each curve shows the average expected reward across all models of one category (reasoning or non-reasoning).

\begin{figure}[htbp]
    \centering

    \begin{subfigure}[b]{0.44\textwidth}
        \centering
        \includegraphics[width=\textwidth]{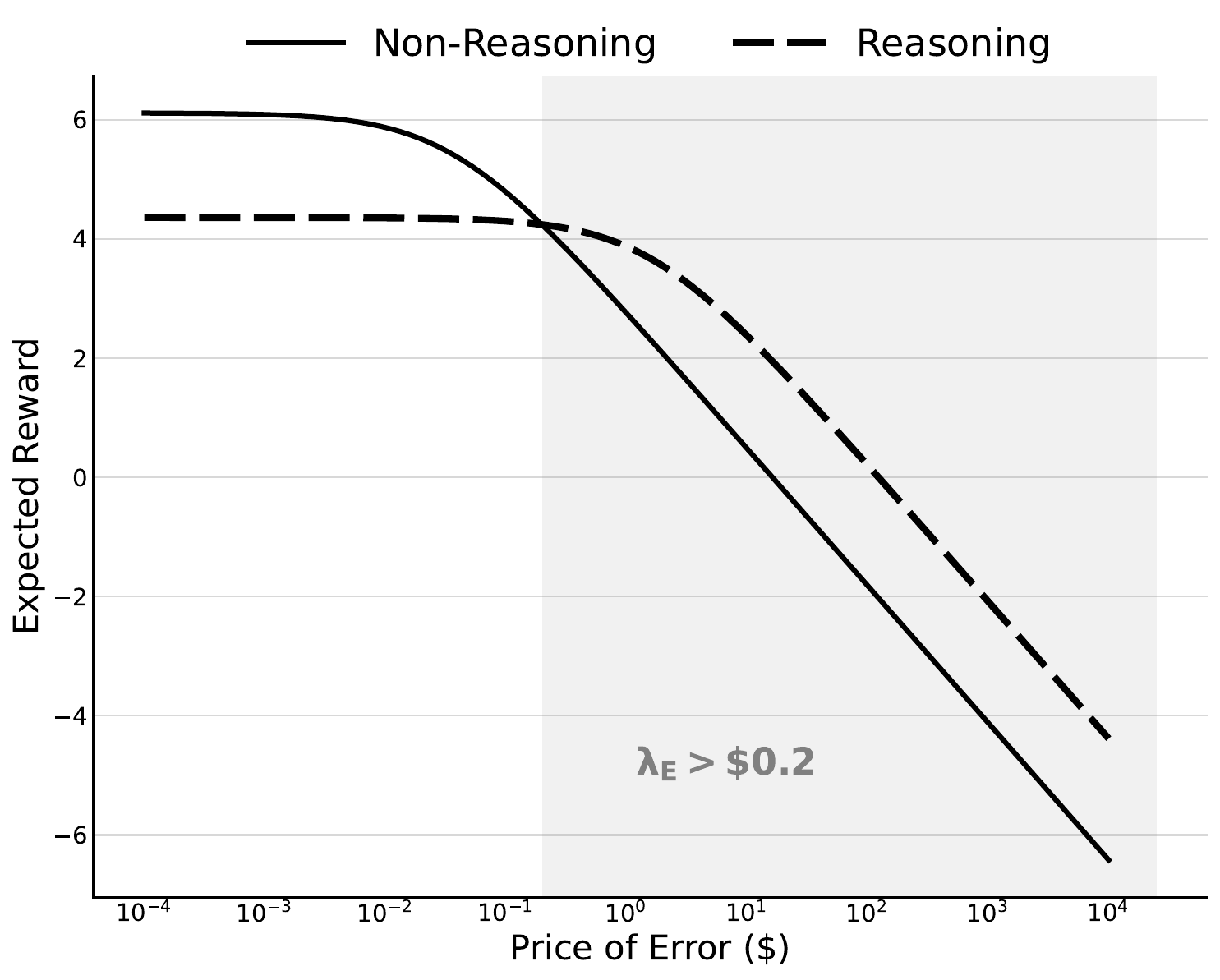}
        \caption{Difficulty: Level 3}
        \label{fig:reasoning_vs_nonreasoning_w_difficulty_3}
    \end{subfigure}
    \hfill
    \begin{subfigure}[b]{0.44\textwidth}
        \centering
        \includegraphics[width=\textwidth]{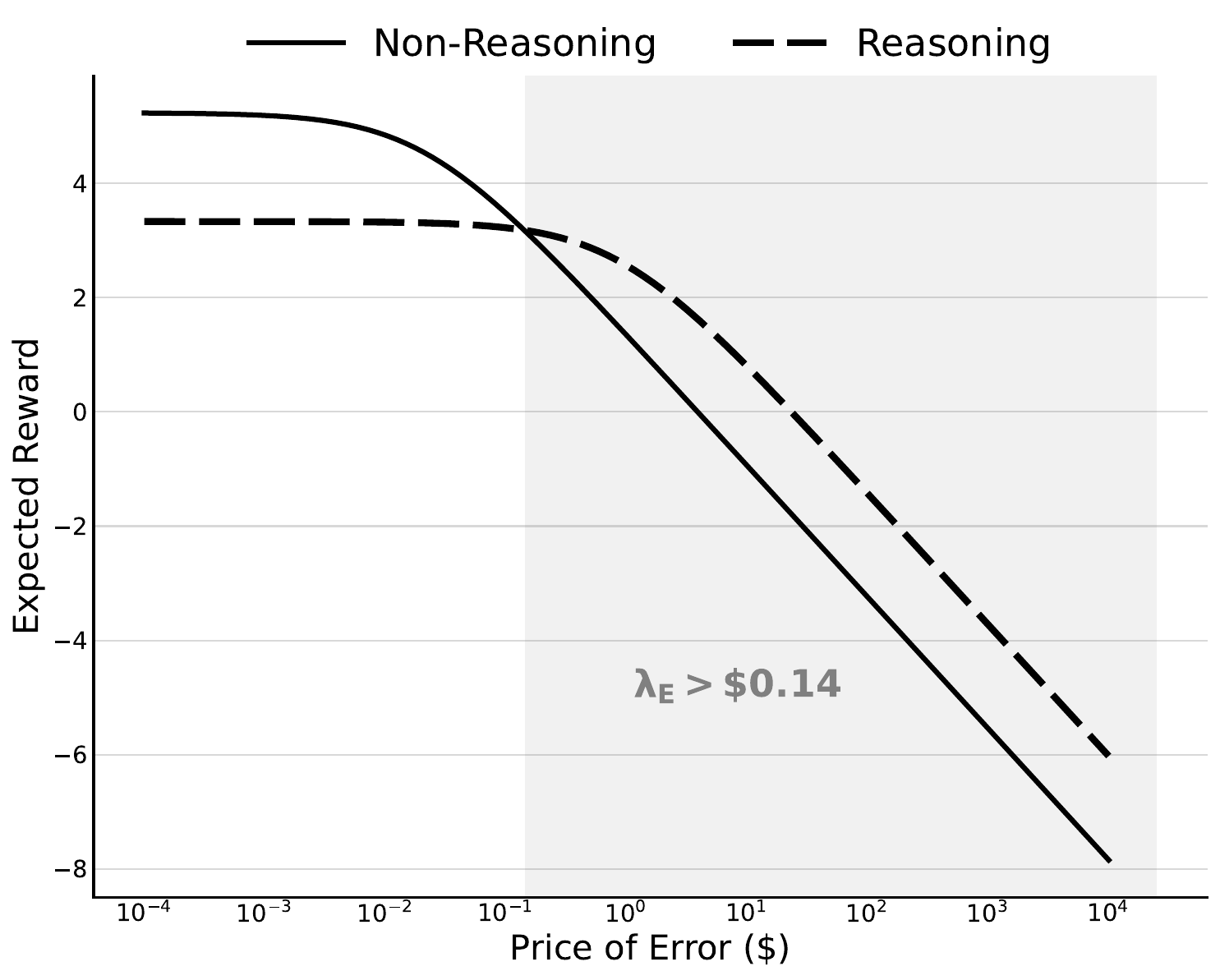}
        \caption{Difficulty: Level 5}
        \label{fig:reasoning_vs_nonreasoning_difficulty_5}
    \end{subfigure}

    \caption{Reasoning models offer superior accuracy-cost trade-offs as soon as the price of error, $\lambda_E$, exceeds \$0.20 per query. The y-axis shows -log(-$R(\lambda_E)$) to make the trends more easily visible.}
    \label{fig:reasoning_vs_nonreasoning}
\end{figure}

The results show that reasoning models offer superior accuracy-cost trade-offs as soon as the cost of making a mistake exceeds \$0.20. This figure is surprisingly low. Suppose a human worker is able to complete the same task in 5 minutes on average, and assume that the consequence of a mistake is having to re-do the task. Then the economic loss of a single mistake exceeds $\$0.20$ as soon as the worker's wages exceed \$2/hour—a number below the U.S. federal minimum wage.

Figure \ref{fig:reasoning_vs_nonreasoning_w_latency} introduces a non-zero price of latency and displays the optimal models across economic constraints, with prices of error $\lambda_E$ ranging from \$0.0001 to \$10,000, and prices of latency $\lambda_L$ ranging from \$0/minute to \$10/minute. These prices of latency correspond to human wages from \$0/hour to \$600/hour. Hence, we believe this range captures a wide variety of use cases for automating meaningful human tasks, ranging from customer support (below \$100/hour) to medical diagnosis (above \$100/hour). We note that this regime does \textbf{not} include the more stringent latency constraints of using LLMs for non-human tasks such as serving popular web applications (\citealp{kohavi2007}) or iterating through a large number of database records. We leave exploration of such tasks—and their higher prices of latency—to future work, as we are most concerned with the emerging practice of automating human tasks using LLMs.

\begin{figure}[htbp]
    \centering

    \begin{subfigure}[b]{0.49\textwidth}
        \centering
        \includegraphics[width=\textwidth]{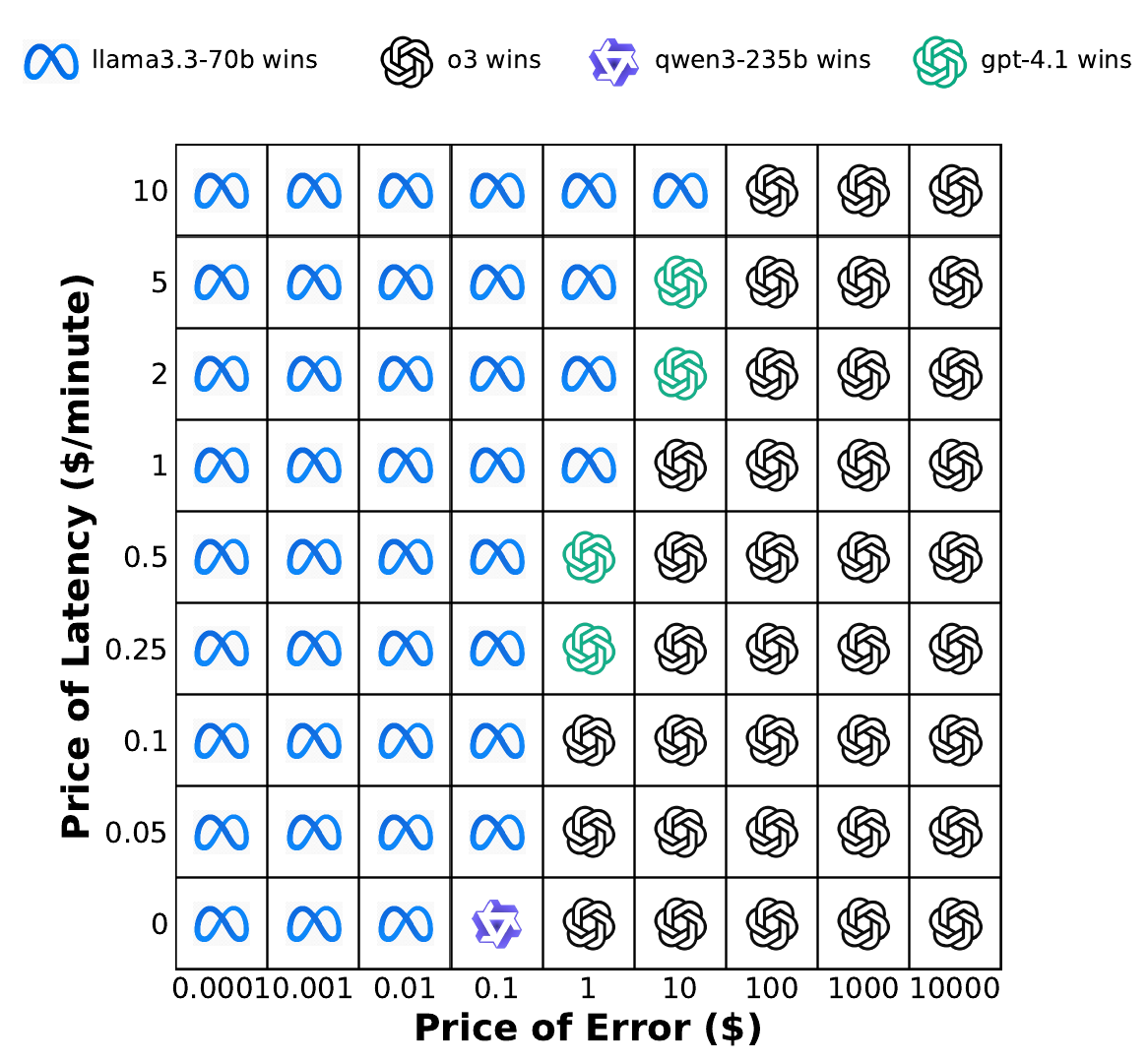}
        \caption{Difficulty: Level 3}
        \label{fig:reasoning_vs_nonreasoning_w_latency_difficulty_3}
    \end{subfigure}
    \hfill
    \begin{subfigure}[b]{0.42\textwidth}
        \centering
        \includegraphics[width=\textwidth]{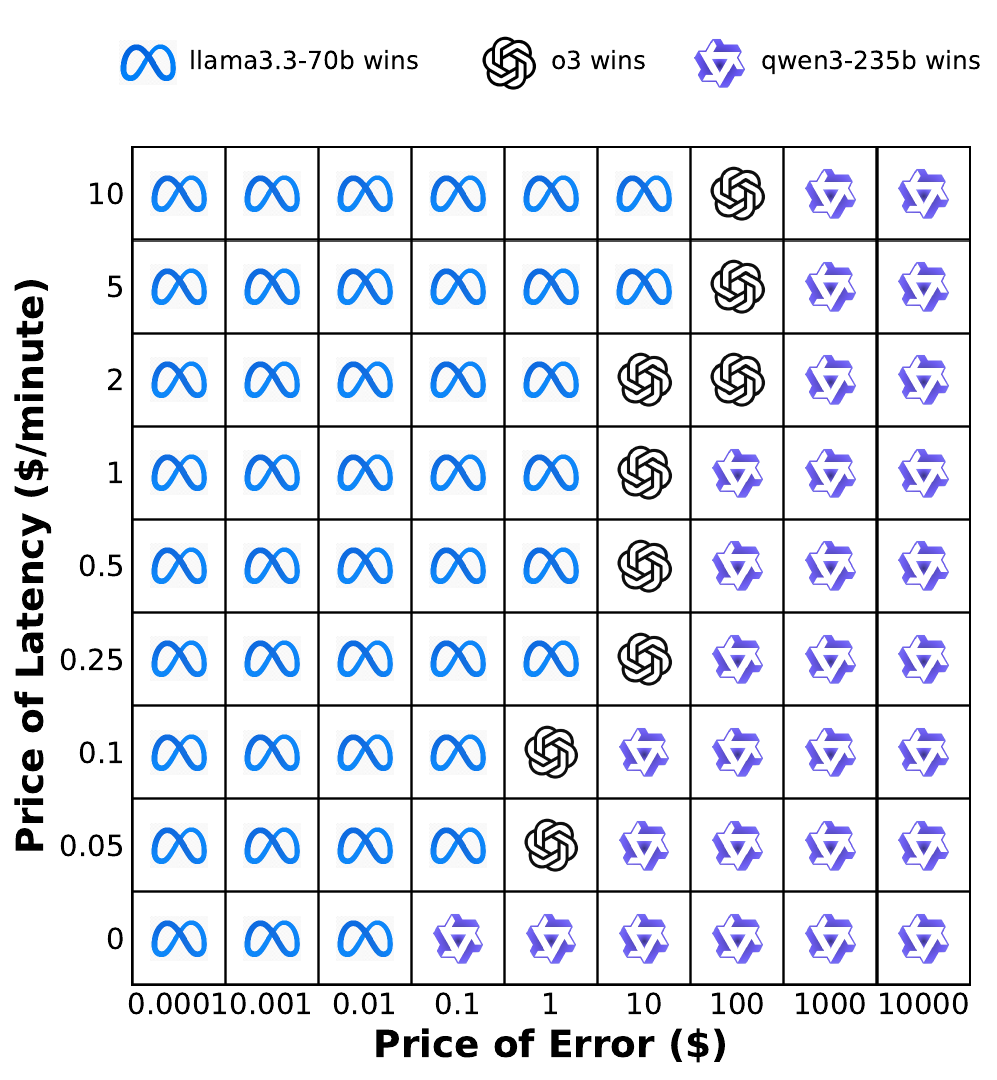}
        \caption{Difficulty: Level 5}
        \label{fig:reasoning_vs_nonreasoning_w_latency_difficulty_5}
    \end{subfigure}

    \caption{Optimal models for different combinations of \textit{price of error} and \textit{price of latency}. Reasoning models show superior performance for prices of error above \$10.}
    \label{fig:reasoning_vs_nonreasoning_w_latency}
\end{figure}

Figure \ref{fig:reasoning_vs_nonreasoning_w_latency} shows that reasoning models generally outperform non-reasoning models for prices of error above \$10 when the price of latency is at most \$5/minute (equivalent to human wages of \$300/hour). For a price of latency of \$10/minute (or \$600/hour), the critical price of error rises to \$100. Among LLMs, Qwen3-235B-A22B, o3, and Llama3.3 70B emerge as the preferred models for the vast majority of economic scenarios.

\subsection{When Does a Single Big Model Outperform a Cascade?}
\label{subsec:cascade_q}

A large language model cascade $\mathcal{M}_\text{small} \rightarrow \mathcal{M}_\text{big}$ sends all queries first to a relatively small model $\mathcal{M}_\text{small}$. If $\mathcal{M}_\text{small}$ is uncertain about the answer, it defers the query to $\mathcal{M}_\text{big}$; otherwise, it directly returns the output. Cascading generally assumes that $\mathcal{M}_\text{big}$ performs better than $\mathcal{M}_\text{small}$ on all queries; hence, we expect that as the price of error $\lambda_E$ increases, there exists a cross-over point $\lambda_E^\text{critical}$ when directly sending queries to the big model $\mathcal{M}_\text{big}$ outperforms cascading.

Figure \ref{fig:single_vs_cascade} compares the performance of three different cascades $\mathcal{M}_\text{small} \rightarrow \mathcal{M}_\text{big}$. For each cascade, $\mathcal{M}_\text{big} = \text{Qwen3 235B-A22B}$, but $\mathcal{M}_\text{small}$ ranges over all the non-reasoning models (Llama3.3 70B, Llama3.1 405B, and GPT-4.1). We evaluate only on the most difficult (level 5) questions of the MATH benchmark. We tune the cascade's deferral threshold on $n=250$ training examples and use the remaining $n=250$ questions for evaluation.

\begin{figure}[htbp]
    \centering

    \begin{subfigure}[b]{0.32\textwidth}
        \centering
        \includegraphics[width=\textwidth]{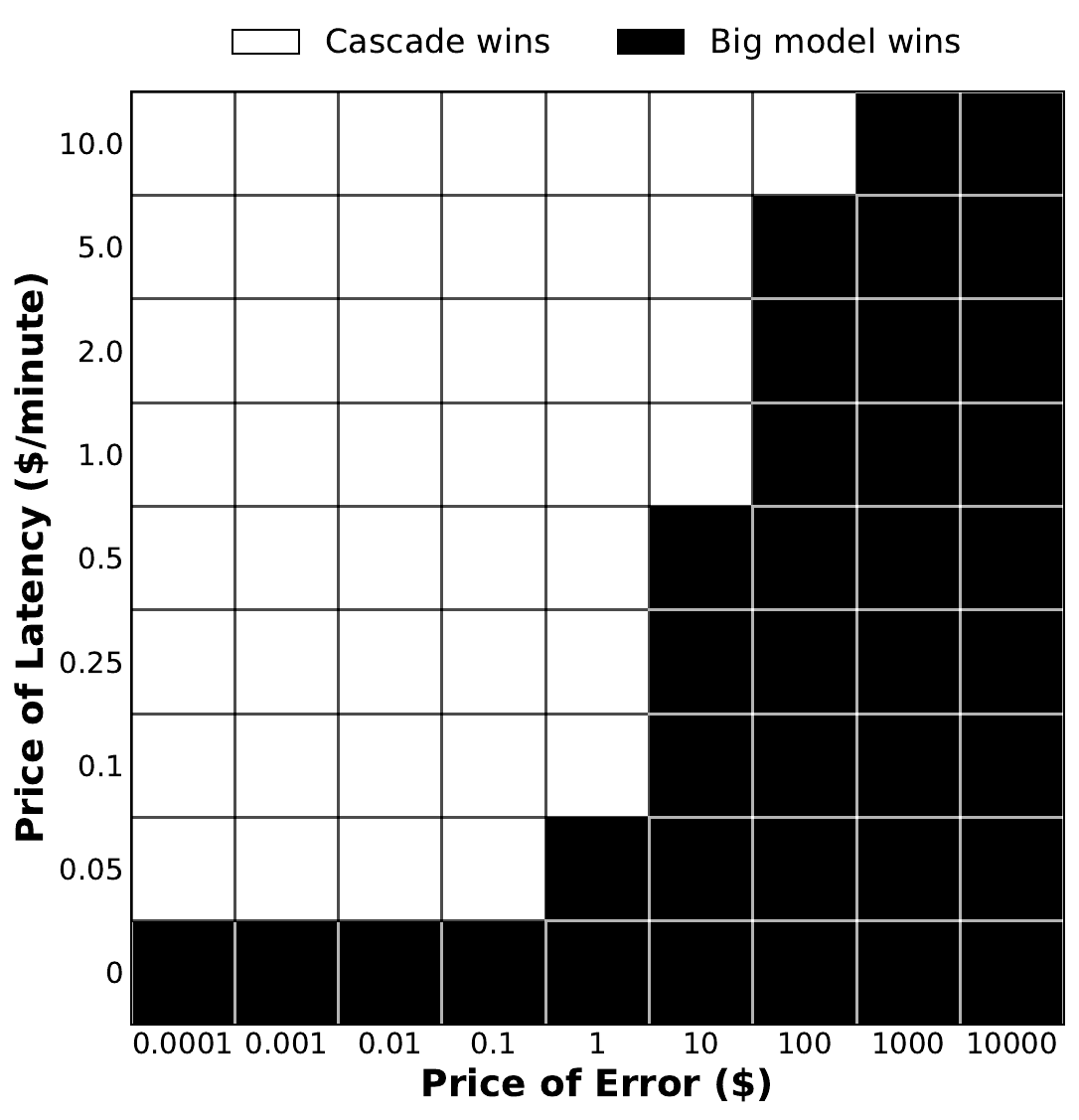}
        \caption{GPT 4.1 $\rightarrow$ $\mathcal{M}_\text{big}$}
        \label{fig:single_vs_cascade_b}
    \end{subfigure}
    \hfill
    \begin{subfigure}[b]{0.32\textwidth}
        \centering
        \includegraphics[width=\textwidth]{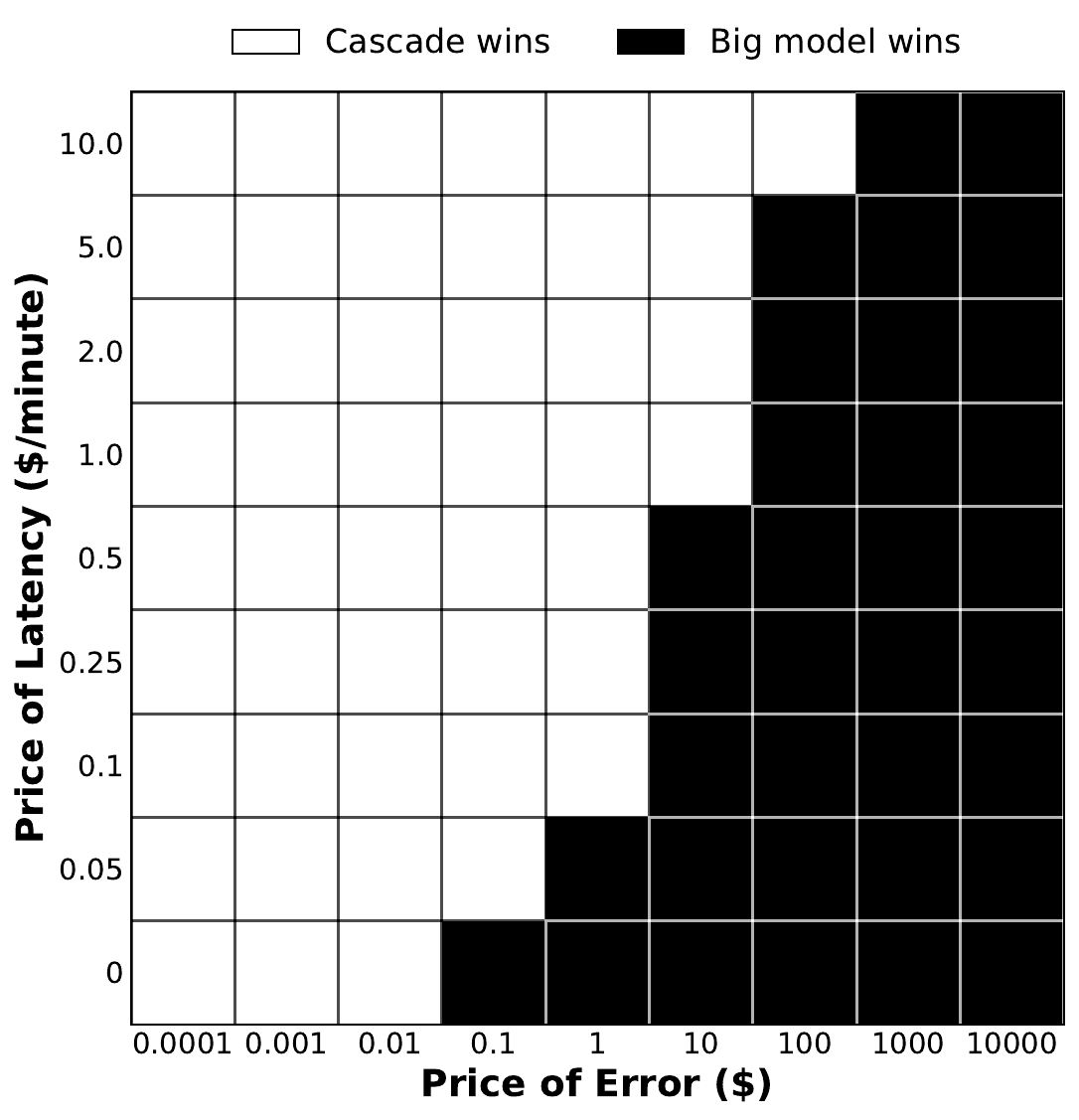}
        \caption{Llama3.3 70B $\rightarrow$ $\mathcal{M}_\text{big}$}
        \label{fig:single_vs_cascade_a}
    \end{subfigure}
    \hfill
    \begin{subfigure}[b]{0.32\textwidth}
        \centering
        \includegraphics[width=\textwidth]{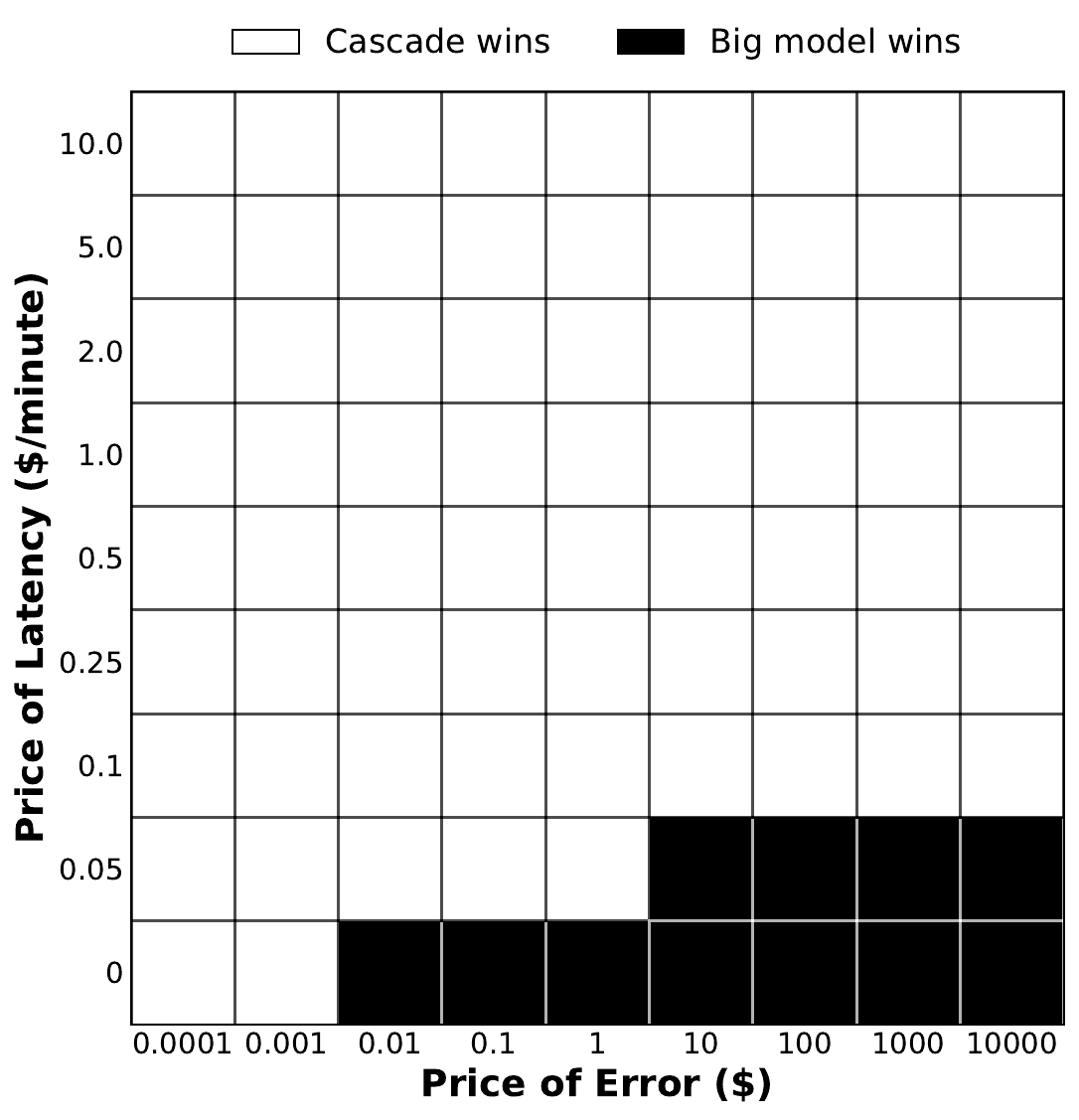}
        \caption{Llama3.3 405B $\rightarrow$ $\mathcal{M}_\text{big}$}
        \label{fig:single_vs_cascade_c}
    \end{subfigure}

    \caption{Directly sending queries to Qwen3 235B-A22B ($\mathcal{M}_\text{big}$) outperforms cascades when the cost of making a mistake exceeds \$0.10 and the price of latency is sufficiently low.}
    \label{fig:single_vs_cascade}
\end{figure}

Disregarding the impact of latency, using $\mathcal{M}_\text{big}$ = Qwen3 235B-A22 as a standalone LLM generally outperforms the cascade $\mathcal{M}_{\text{small}} \rightarrow \mathcal{M}_\text{big}$—as soon as the price of error exceeds $\lambda_E > \$0.10$. As the price of latency increases to $\lambda_L = \$0.5/\text{minute}$ (equiv. to \$30/hour), the critical price of error $\lambda_E^\text{critical}$ ratchets up to \$10. At $\lambda_L = \$10/\text{minute}$ (equiv. to \$600/hour), $\lambda_E^\text{critical}$ increases to $\$1,000$. This finding suggests that when automating medical diagnosis (with an estimated price of error between \$100 and \$1,000, as discussed in Section \ref{subsec:poe_estimate}), it may be preferable to avoid cascading.

However, this analysis carries an important caveat. For the cascade with $\mathcal{M}_\text{small}$ = Llama3.3 405B, we observed a marked and surprising outperformance over the other cascades—even though as a standalone LLM, Llama3.3 405B performs strictly worse than Llama3.3 70B, as shown in Figure \ref{fig:baseline_performance_difficulty_5}. The cascade $\text{Llama3.1 405B} \rightarrow \text{Qwen3 235B-A22B}$ yields better accuracy-cost-latency trade-offs than Qwen3 235B-A22B for the vast majority of economic scenarios, including prices of error up to $\$10,000$ and prices of latency up to \$10/minute (equiv. to \$600/hour).

Why does using Llama3.1 405B as the small model result in superior cascade performance? To explain this phenomenon, we point to the model's remarkable self-verification performance. To lay out our argument, we first restate the formula for the error rate of a two-model cascade given without proof by \citet{zellinger2024}:
\begin{theorem}[Cascade Error]
    \label{thm:cascade_error}
    Consider a cascade $\mathcal{M}_\text{small} \rightarrow \mathcal{M}_\text{big}$, where the deferral decision of $\mathcal{M}_\text{small}$ is determined by the indicator $\mathds{1}_D$. Then the error rate, $e_\text{cascade}$ of the cascade is
    \begin{equation}
        e_{\mathcal{M}_\text{small} \rightarrow \mathcal{M}_\text{big}} = (1 - p_d)~e_{\mathcal{M}_\text{small}} + p_d~e_{\mathcal{M}_\text{big}} + \text{Cov}(\mathds{1}_D, \mathds{1}_\text{error}^{\mathcal{M}_\text{big}}) - \text{Cov}(\mathds{1}_D, \mathds{1}_\text{error}^{\mathcal{M}_\text{small}}),
    \end{equation}
    where $p_d := \mathbb{E}[\mathds{1}_D]$ is the deferral rate, $e_{\mathcal{M}_\text{small}} := \mathbb{E}[\mathds{1}_\text{error}^{\mathcal{M}_\text{small}}]$ is the error rate of $\mathcal{M}_\text{small}$, and $e_{\mathcal{M}_\text{big}} := \mathbb{E}[\mathds{1}_\text{error}^{\mathcal{M}_\text{big}}]$ is the error rate of $\mathcal{M}_\text{big}$.
\end{theorem}

\begin{proof}
    This result follows from writing out the formula for cascade error,
    \begin{equation}
        e_{\mathcal{M}_\text{small} \rightarrow \mathcal{M}_\text{big}} = \mathbb{E}[ (1-\mathds{1}_D) \mathds{1}_\text{error}^{\mathcal{M}_\text{small}} + \mathds{1}_D \mathds{1}_\text{error}^{\mathcal{M}_\text{big}} ].
    \end{equation}
    Using the linearity of expectation, followed by adding and subtracting terms to recover the covariances, gives the result.
\end{proof}

Theorem \ref{thm:cascade_error} shows that the error rate of a cascade $\mathcal{M}_\text{small} \rightarrow \mathcal{M}_\text{big}$, relative to randomly sending queries to $\mathcal{M}_\text{small}$ and $\mathcal{M}_\text{big}$, depends on the difference in covariances $\text{Cov}(\mathds{1}_D, \mathds{1}_\text{error}^{\mathcal{M}_\text{big}}) - \text{Cov}(\mathds{1}_D, \mathds{1}_\text{error}^{\mathcal{M}_\text{small}})$. Intuitively, these covariances measure the increase in the models' error rates conditional on deferral. Specifically, $\text{Cov}(\mathds{1}_D, \mathds{1}_\text{error}^{\mathcal{M}_\text{small}})$ expresses the agreement between the deferral decision and the true uncertainty of $\mathcal{M}_\text{small}$.\footnote{Thus, $\text{Cov}(\mathds{1}_D, \mathds{1}_\text{error}^{\mathcal{M}_\text{small}})$ is closely related to the area under the accuracy-rejection curve (AUARC) from selective prediction (\citealp{elyaniv2010}).} In practice, $\text{Cov}(\mathds{1}_D, \mathds{1}_\text{error}^{\mathcal{M}_\text{big}}) \ll \text{Cov}(\mathds{1}_D, \mathds{1}_\text{error}^{\mathcal{M}_\text{small}})$, so $\text{Cov}(\mathds{1}_D, \mathds{1}_\text{error}^{\mathcal{M}_\text{small}})$ alone is a strong indicator of cascade effectiveness. We refer to this metric as the \textit{cascade error reduction} (CER).

\begin{figure}[htbp]
    \centering

    \begin{subfigure}[b]{0.32\textwidth}
        \centering
        \includegraphics[width=\textwidth]{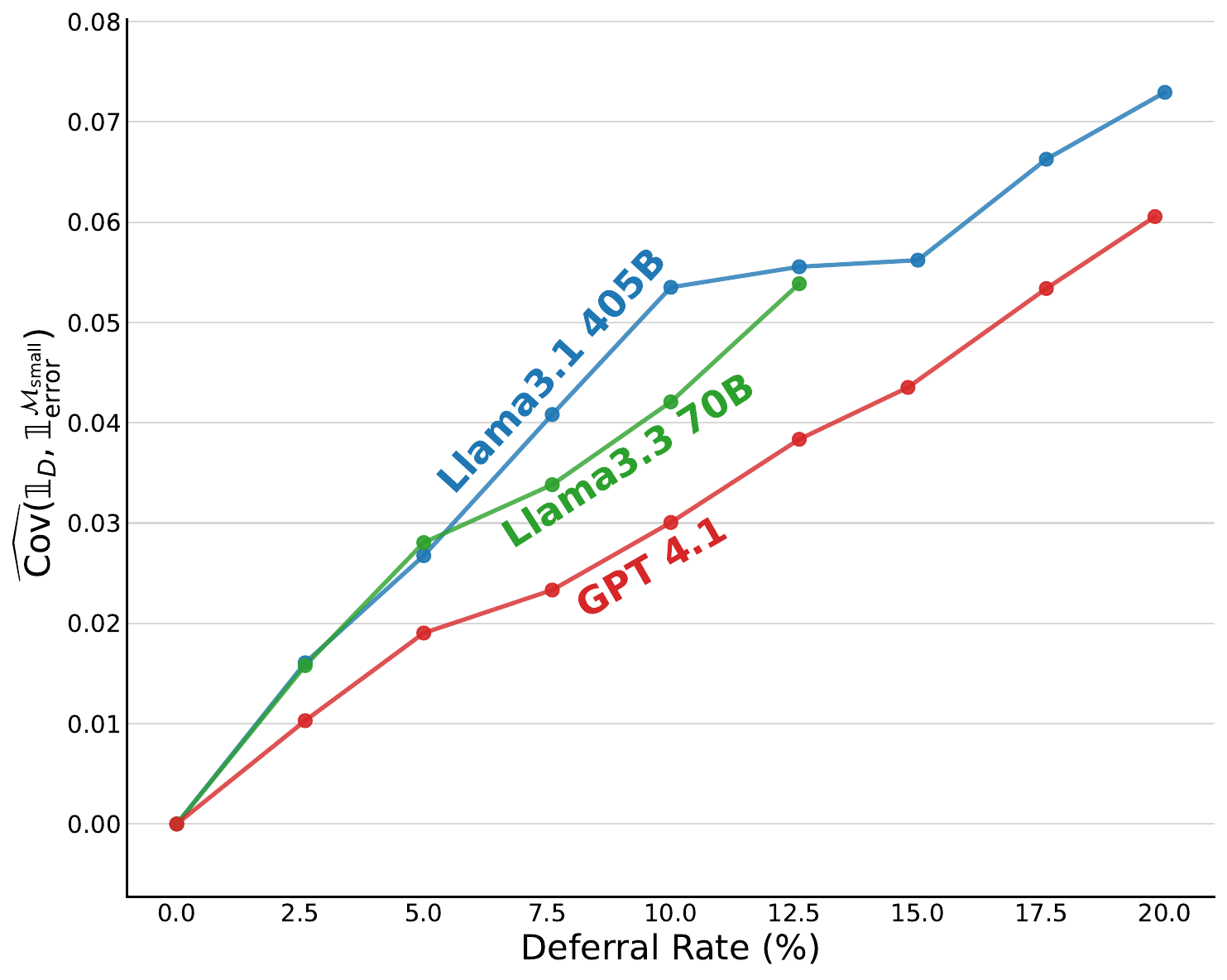}
        \caption{CER}
        \label{fig:why_llama405_makes_a_great_cascade_a}
    \end{subfigure}
    \hfill
    \begin{subfigure}[b]{0.32\textwidth}
        \centering
        \includegraphics[width=\textwidth]{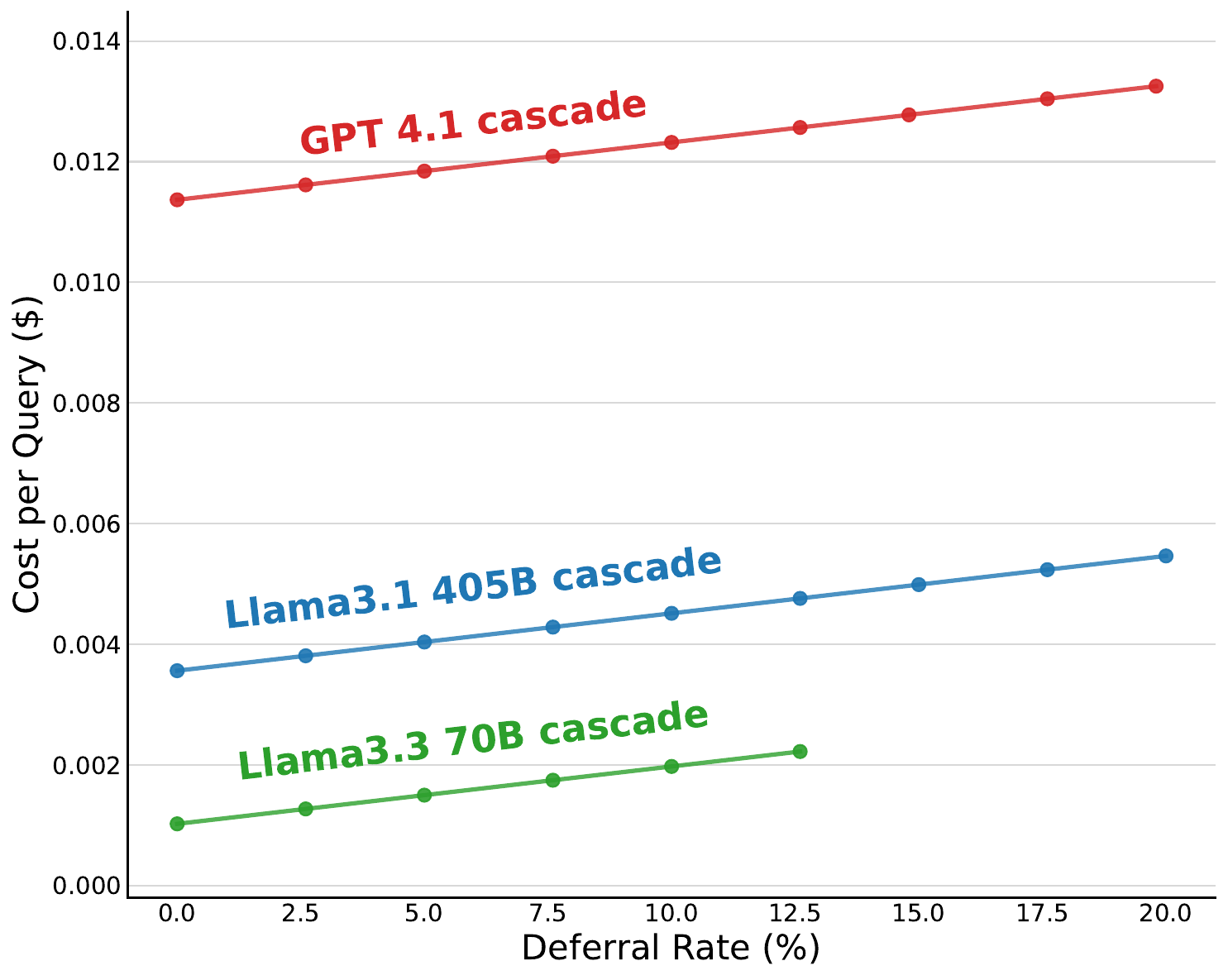}
        \caption{Cost}
        \label{fig:why_llama405_makes_a_great_cascade_b}
    \end{subfigure}
    \hfill
    \begin{subfigure}[b]{0.32\textwidth}
        \centering
        \includegraphics[width=\textwidth]{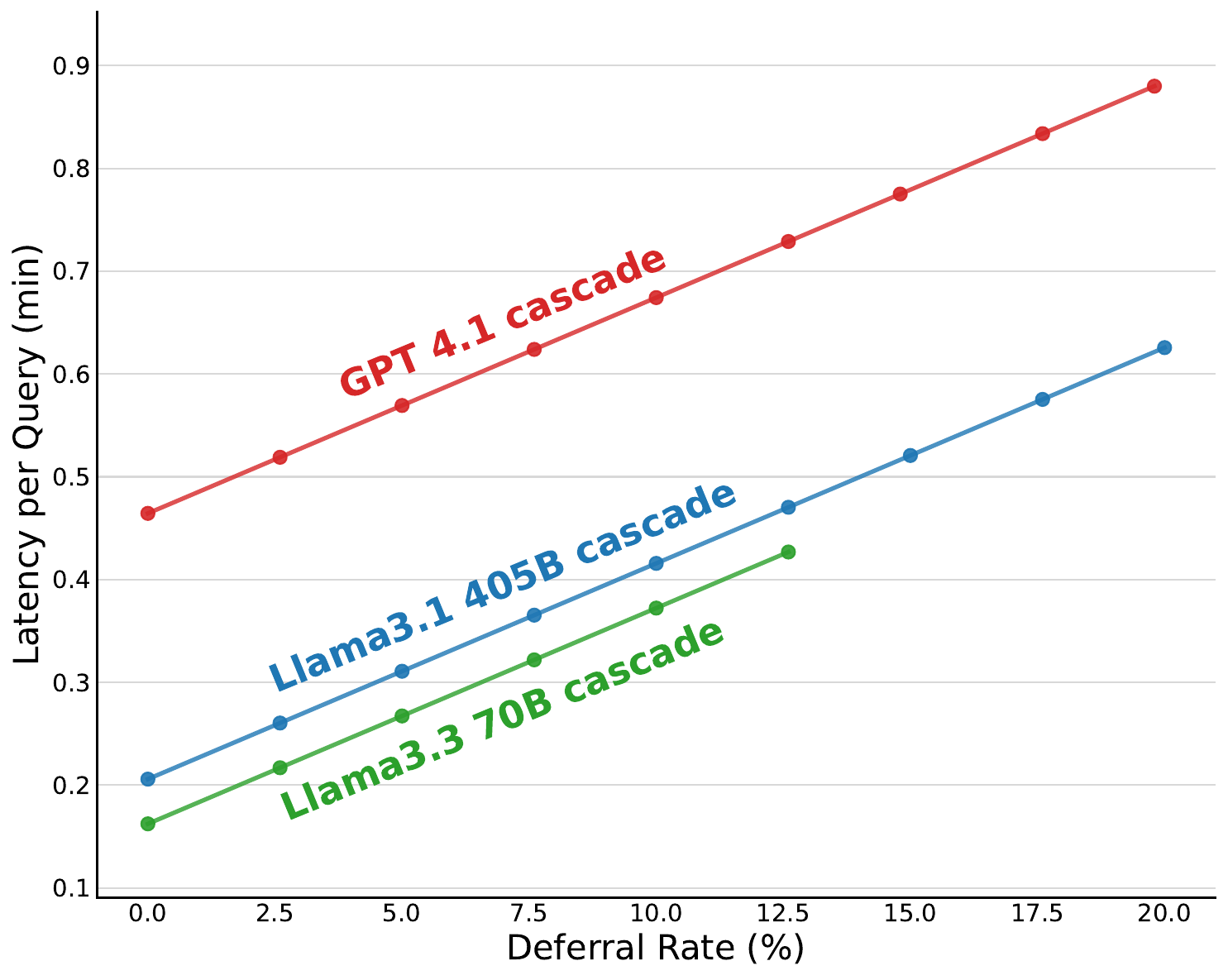}
        \caption{Latency}
        \label{fig:why_llama405_makes_a_great_cascade_c}
    \end{subfigure}

    \caption{Using Llama3.1 405B as $\mathcal{M}_\text{small}$ yields superior cascading performance, despite its subpar performance as a standalone model, because it performs better at self-verification—indicated by a higher cascade error reduction (CER) (a). In contrast to error rate, cost and latency are simple linear functions of the deferral rate (b and c).}
    \label{fig:why_llama405_makes_a_great_cascade}
\end{figure}

Figure \ref{fig:why_llama405_makes_a_great_cascade} plots CER = $\widehat{\text{Cov}}(\mathds{1}_D, \mathds{1}_\text{error}^{\mathcal{M}_\text{small}})$ against the cascades' deferral thresholds. The figure shows that $\widehat{\text{Cov}}(\mathds{1}_D, \mathds{1}_\text{error}^{\mathcal{M}_\text{small}})$ is highest for Llama3.1 405B, explaining its superior suitability for playing the role of $\mathcal{M}_\text{small}$ in a cascade (Figure \ref{fig:single_vs_cascade}). Figures \ref{fig:why_llama405_makes_a_great_cascade_b} and \ref{fig:why_llama405_makes_a_great_cascade_c} plot the cascades' costs and latencies for the same range of deferral rates, verifying that Llama3.1 405B's superior cascade performance is driven by error rate, not cost or latency.

\section{Related Work}

In this section, we discuss related work.

\noindent \textbf{Economic Analysis of LLMs.} Several papers have explored the economic and labor market impacts of large language models (\citealp{brynjolfsson2023}; \citealp{eloundou2024}). For example, \citet{eloundou2024} assess the vulnerability of different jobs to AI automation, finding that highly paid professional work is particularly exposed. Different from our work, the authors define exposure with respect to achieving time savings while maintaining the same quality of work. By contrast, we wish to highlight the comparatively low cost of AI labor, as reflected in costs per query between \$0.01 and \$0.1, even for state-of-the-art reasoning LLMs. Our argument is that lowering the cost of labor—even without saving time—yields significant economic benefits.

When it comes to ranking the performance of LLMs based on economic principles, the closest work to ours is the preprint by \citet{erol2025}. The authors propose quantifying accuracy-cost trade-offs by the ratio of cost to accuracy. Numerically, this approach is mathematically similar to our multi-objective reward (\ref{eq:per-query-reward-concrete}) when considering only accuracy and cost, and using a constant, unchanging price of error for all use cases. We believe that our framework's ability to handle a greater number of simultaneous performance objectives, and to differentiate between the economic realities of different industries (for example, customer support vs medical diagnosis), to be significant strengths. 

Moreover, we consider the theoretical motivation of \citet{erol2025}'s cost-of-pass metric to be flawed. Their argument is that with an accuracy of $a$, an LLM can be viewed to generate the correct answer to a query in $1/a$ attempts (assuming a geometric distribution). Hence, their metric measures the total cost (``cost of pass'') required to produce the correct answer. However, the correctness of repeated samples from an LLM with accuracy $a$ may \textit{only} be modeled as i.i.d. Bernoulli(a) trials if we also randomly sample new queries $q$. For a fixed query, repeated sampling does not reliably yield a correct answer, although it offers some benefits (\citealp{chen2024}).

\noindent \textbf{LLM Evaluation and Benchmarking}. The evaluation of large language models has received significant attention over the past few years (\citealp{laskar2024}). Prior work has mostly focused on assessing specific capabilities of LLMs, such as summarization (\citealp{narayan2018}), general knowledge (\citealp{hendrycks2017}), truthfulness (\citealp{lin2022}), mathematical reasoning (\citealp{hendrycks2021}), and others. In general, individual benchmarks rapidly become obsolete as LLMs become more capable. Notably, \citet{liang2023} provide a compelling synthesis of many of the different capabilities worth evaluating. In addition to benchmarking, the Chatbot Arena (\citealp{chiang2024}) has popularized ranking LLMs through crowd-sourced pairwise comparisons, in a manner similar to reinforcement learning from human feedback (\citealp{ouyang2022}).

These evaluation efforts are not directly comparable to our work, as we focus on \textit{multi-objective} evaluation of conflicting performance objectives, such as accuracy, cost, and latency. It is not sufficient to evaluate these metrics individually, since we wish to simultaneously optimize for these performance objectives.

\noindent \textbf{Multi-Objective LLM Systems.} Our emphasis on multi-objective evaluation is reflected in research on multi-LLM systems such as cascades (\citealp{ding2024}; \citealp{chen2023}; \citealp{madaan2024}) and routers (\citealp{shnitzer2023}; \citealp{hari2023}; \citealp{jitkrittum2025}). Researchers have mainly framed the accuracy-cost trade-off as a constrained minimization of the error rate subject to a cost budget (\citealp{chen2023}; \citealp{jitkrittum2024}; \citealp{hu2024}). Evaluation of different LLM systems proceeds by plotting Pareto-optimal error rates against the corresponding cost budgets. Unfortunately, this approach does not easily generalize to more than two performance objectives, since it is difficult to compare the quality of higher-dimensional Pareto frontiers (\citealp{zellinger2025b}). Specifically, the volume under a Pareto surface is not a meaningful metric unless the lower-dimensional projections of different Pareto frontiers substantially intersect.\footnote{The same problem may occur in two dimensions. Consider error-cost curves that are horizontally shifted, i.e., the attainable cost budgets are disjoint. Comparing the areas under these curves does not yield meaningful results.}

In addition, we suspect that prior work's emphasis on setting budgets for cost or latency may be overly influenced by the peculiarities of the IT industry, as these budgets closely reflect that industry's \textit{service-level objectives} (SLO). However, artificial intelligence is a society-wide phenomenon that extends far beyond IT. As AI starts to perform meaningful human work, it will likely emerge as a revenue generator rather than a cost center—diminishing the relevance of cost budgets as a mental framework.

\noindent \textbf{Computation of Pareto Frontiers}. From an optimization perspective, our methodology corresponds to a well-known method for approximating Pareto frontiers called the \textit{weighted sum method} for \textit{scalarization} (\citealp{koski1988}; \citealp{jahn1991}; \citealp{banholzer2019}). Prior work has identified certain disadvantages of scalarization: for example, it may not yield all Pareto-optimal trade-offs when the Pareto surface is non-convex, or coverage of the Pareto frontier might be unevenly distributed (\citealp{das1997}). Such issues are not a major concern for us, as we motivate our methodology on economic grounds, by casting LLM systems as reward-maximizing agents (see Section \ref{subsec:llms-as-agents}). We note, however, that Pareto frontiers arising in performance evaluation of LLM systems tend to be convex since randomly routing queries between two LLMs smoothly interpolates between their respective performance metrics.

\section{Conclusion}

We have presented an economic framework for evaluating the performance of LLMs and LLM systems. Compared to plotting Pareto frontiers, our approach yields a single optimal model based on a use case's economic constraints: the cost of making a mistake (\textit{price of error}), the cost of incremental latency (\textit{price of latency}), and the cost of abstaining from a query (\textit{price of abstention}), as well as possible additional objectives such as privacy (\citealp{zhang2025}).

We motivated our framework by casting LLMs and LLM systems as reward-maximizing agents, and revealed theoretical relationships between our proposed methodology and the established notion of Pareto optimality.

Applying our framework to empirically exploring the practical relevance of non-reasoning LLMs and cascades, we found several interesting results. First, reasoning models offer superior accuracy-cost trade-offs on difficult mathematics questions as soon as the price of error exceeds \$0.01. Second, a single large LLM $\mathcal{M}_\text{big}$ typically outperforms a cascade $\mathcal{M}_\text{big} \rightarrow \mathcal{M}_\text{big}$ for prices of error as low as \$0.1. 

Extrapolating these findings from mathematics to other domains, our results carry significant economic implications. We recommend that when automating meaningful human tasks with AI models, practitioners should typically use the most powerful available model, rather than attempt to minimize inference costs, since inference costs are likely dwarfed by the economic impact of AI errors. Fundamentally, this recommendation is based on the low costs per query of LLMs, which are increasingly negligible compared to human wages.

\section{Limitations}

First, our results on the MATH benchmark may be affected by data contamination, since the LLMs we evaluate may have been trained on similar questions. In addition, using Llama3.1 405B for correctness evaluation may artificially inflate this model's self-verification accuracy—however, we believe that this latter effect is limited since the prompt for correctness evaluation relies on access to the ground truth reference answer, whereas self-verification only incorporates the model's proposed (but possibly incorrect) answer.

\nocite{srivastava2023}

\bibliography{main}

\newpage

\section*{Appendix A - Scaling of Output Tokens with Query Difficulty}

\begin{figure}[htbp]
    \centering

    \begin{subfigure}[b]{0.32\textwidth}
        \centering
        \includegraphics[width=\textwidth]{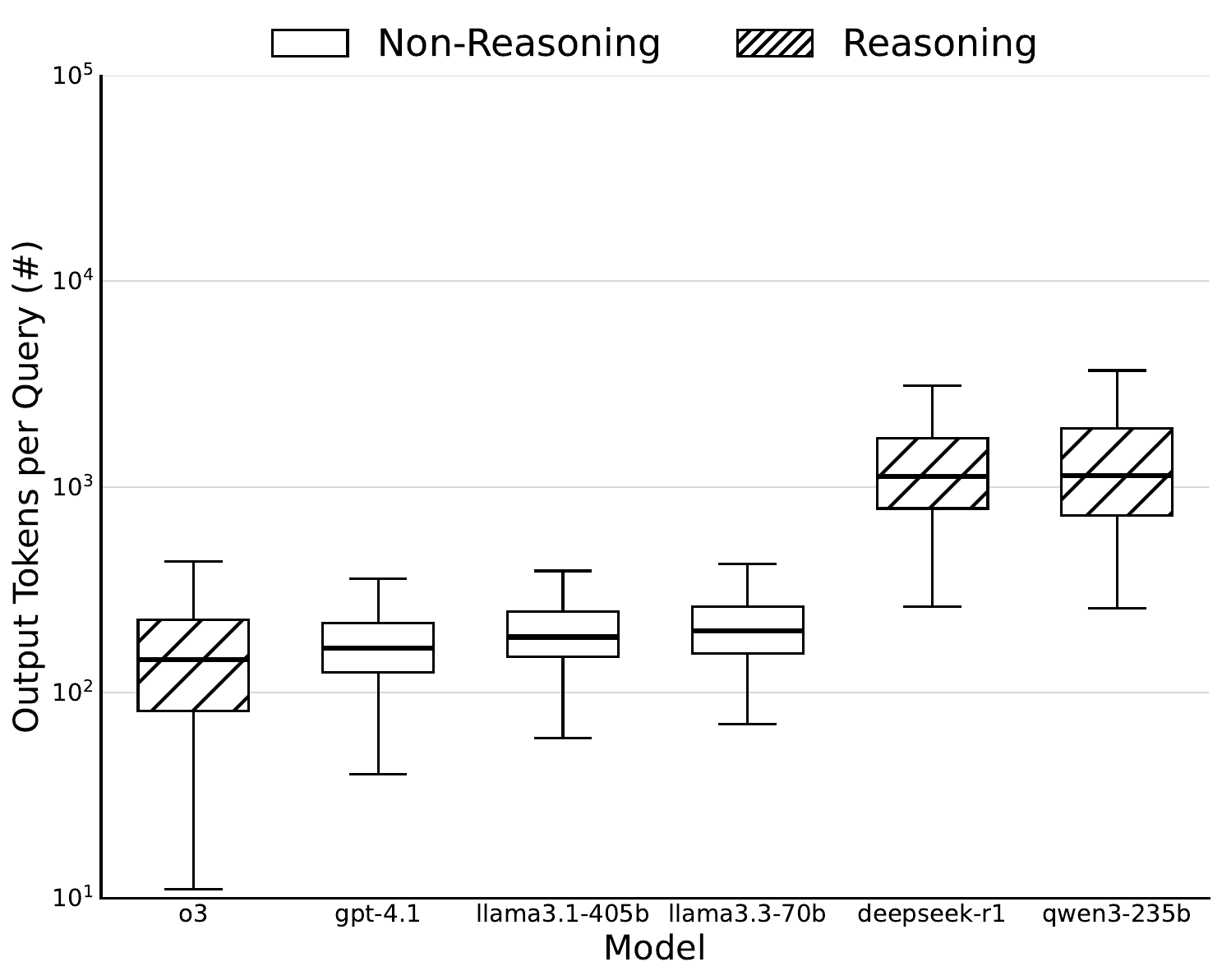}
        \caption{Difficulty: Easy}
        \label{fig:baseline_performance_difficulty_1_tokens}
    \end{subfigure}
    \hfill
    \begin{subfigure}[b]{0.32\textwidth}
        \centering
        \includegraphics[width=\textwidth]{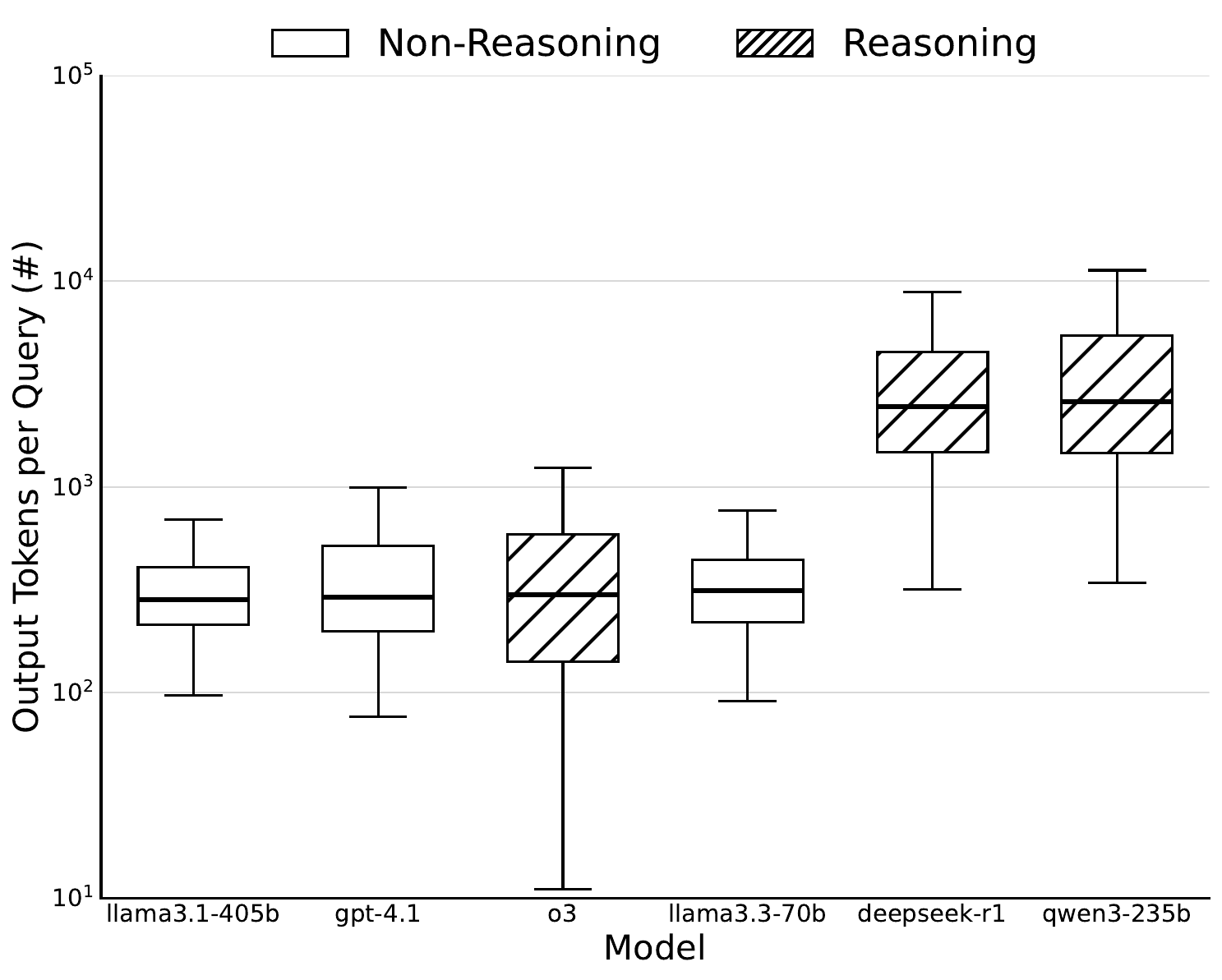}
        \caption{Difficulty: Medium}
        \label{fig:baseline_performance_difficulty_3_tokens}
    \end{subfigure}
    \hfill
    \begin{subfigure}[b]{0.32\textwidth}
        \centering
        \includegraphics[width=\textwidth]{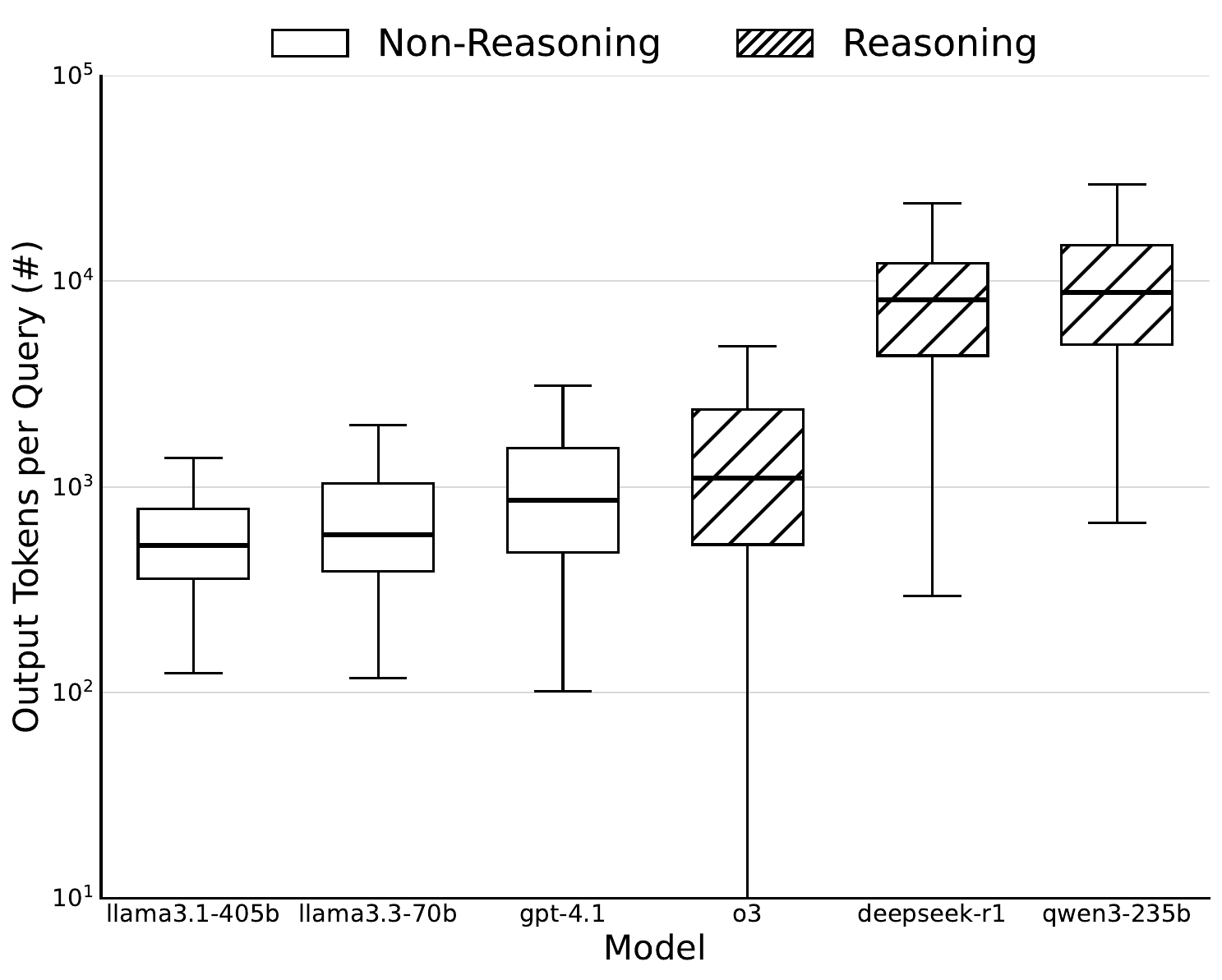}
        \caption{Difficulty: Hard}
        \label{fig:baseline_performance_difficulty_5_tokens}
    \end{subfigure}

    \caption{Non-reasoning models prompted with chain-of-thought exhibit similar scaling of output tokens compared to reasoning models, but reasoning models start from a higher baseline number of output tokens.}
    \label{fig:baseline_performance_tokens}
\end{figure}

\section*{Appendix B - Proof of Theorem 1}

\begin{theorem*}
\label{thm:pareto_mapping_appendix}
Let $\theta^{*}(\lambda)$ be the solution to the reward maximization problem
\begin{equation}
\label{eq:reward_opt_appendix}
\begin{aligned}
    \theta^{*} =~ & \text{argmax}_\theta && R(\lambda; \theta),
\end{aligned}
\end{equation}
where $\theta \in \mathbb{R}^{p}$ denotes an LLM system's tunable parameters, and $\lambda$ is the vector of economic costs as defined in Section \ref{subsec:quantifying_user_preferences}. Assume that regularity conditions hold, such that for each $\lambda \in \mathbb{R}_{>0}^{|\mathcal{P}_\text{numeric}| + |\mathcal{P}_\text{binary}|}$ there exist bounds $\{ \gamma_\mu \}_{\mu \in \mathcal{P}_\text{numeric}}$ and $\{ \gamma_\chi \}_{\mu \in \mathcal{P}_\text{binary}} > 0$ such that $\theta^{*}(\lambda)$ is equivalently the solution of the constrained optimization problem
\begin{equation}
\label{eq:constrained_opt_general}
\begin{aligned}
    \theta^{*} =~ & \text{argmin}_\theta && \hat{\mathbb{E}}_\theta[C] \\
    & \text{subject to} && \hat{\mathbb{E}}_\theta[\mu] \leq \gamma_\mu, ~~\mu \in \mathcal{P}_\text{numeric} \\
    & && \hat{\mathbb{E}}_{\theta}[\mathds{1}_{\chi}] \leq \gamma_\chi, ~~\chi \in \mathcal{P}_\text{binary},
\end{aligned}
\end{equation}
and vice versa for $\gamma \mapsto \lambda(\gamma)$. Then the vector of economic costs, $\lambda$, maps surjectively onto the Pareto surface via the mapping
\begin{equation}
\label{eq:pareto_mapping}
    \lambda \mapsto (\hat{\mathbb{E}}_{\theta^{*}(\lambda)}[C], \hat{\mathbb{E}}_{\theta^{*}(\lambda)}[\mu_1], ..., \hat{\mathbb{E}}_{\theta^{*}(\lambda)}[\mu_{|\mathcal{P}_\text{numeric}|}], \hat{\mathbb{P}}_{\theta^{*}(\lambda)}[\chi_1], ..., \hat{\mathbb{P}}_{\theta^{*}(\lambda)}[\chi_{|\mathcal{P}_\text{binary}|}]).
\end{equation}
\end{theorem*}

\begin{proof}
First, we show that $\lambda$ maps to the Pareto surface. Second, we show that this mapping is surjective.

Consider any $\lambda \in \mathbb{R}_{>0}^{|\mathcal{P}_\text{numeric}| + |\mathcal{P}_\text{binary}|}$, and let $x = (\hat{\mathbb{E}}_{\theta^{*}(\lambda)}[C], ..., \hat{\mathbb{P}}_{\theta^{*}(\lambda)}[\chi_{|\mathcal{P}_\text{binary}|}])$. Suppose for the sake of contradiction that $x$ is not Pareto optimal. Then there exist $\theta'$ and $x' = (\hat{\mathbb{E}}_{\theta'(\lambda)}[C], ..., \hat{\mathbb{P}}_{\theta'(\lambda)}[\chi_{|\mathcal{P}_\text{binary}|}])$ such that $x'$ dominates $x$. It follows that
\begin{align}
    R(\lambda; \theta') - R(\lambda; \theta^{*}(\lambda)) & = \mathbb{E}_{\theta^{*}(\lambda)} [C] - \mathbb{E}_{\theta'} [C] \\
    & ~~~~+ \sum_{\mu} \lambda_\mu (\mathbb{E}_{\theta^{*}(\lambda)} [\mu] - \mathbb{E}_{\theta'} [\mu]) + \sum_{\chi} \lambda_\chi (\mathbb{P}_{\theta^{*}(\lambda)}(\chi) - \mathbb{P}_{\theta'} (\chi)) \\
    & > 0,
\end{align}
contradicting the optimality of $\theta^{*}(\lambda)$.

Now consider any $x \in \mathbb{R}^{|\mathcal{P}_\text{numeric}| + |\mathcal{P}_\text{binary}| + 1}$ on the Pareto surface. Let $\theta^{*}_c$ be the solution of the constrained optimization problem (\ref{eq:constrained_opt_general}) with $\{ \gamma_\mu \}$, $\{ \gamma_\chi \}$ equal to the last $|\mathcal{P}_\text{numeric}| + |\mathcal{P}_\text{binary}|$ components of $x$. Let $x^{*} = (\hat{\mathbb{E}}_{\theta^{*}_c}[C], ..., \hat{\mathbb{P}}_{\theta^{*}_c}[\chi_{|\mathcal{P}_\text{binary}|}])$.

We argue that $x = x^{*}$. Indeed, $x$ is a feasible point of $(\ref{eq:constrained_opt_general})$ since it lies on the Pareto surface and meets the inequality constraints. In addition, $x^{*}_0 = \hat{\mathbb{E}}_{\theta^{*}_c}[C]$ cannot be \textit{less} than $x_0$, since otherwise $x^{*}$ would dominate $x$. Furthermore, $x^{*}_0 = \hat{\mathbb{E}}_{\theta^{*}_c}[C]$ cannot be \textit{greater} than $x_0$ by the optimality of $x^{*}$, since $x$ is feasible. So $x^{*}_0 = x_0$. Hence, $x^{*} \leq x$ componentwise. Hence, the remaining components of $x^{*}$ and $x$ must be equal, since otherwise $x^{*}$ would dominate $x$. So $x^{*} = x$.

Now observe that $x^{*}$ arises from solving (\ref{eq:reward_opt}) with $\lambda = \lambda(\gamma)$. Hence, $\lambda$ maps to $x$ under the mapping (\ref{eq:pareto_mapping}).
\end{proof}

\section*{Appendix C - Models and Pricing}

Table~\ref{tab:llm-pricing} lists the large language models (LLM) we used in our experiments and the API prices at the time of our experiments (June 2025). We provide each model's exact API identifier together with the cost for input and output tokens.

We used the default hyperparameters (temperature, top-p, top-k, etc.) for sampling from the LLMs, except that we raised the maximum number of output tokens to 100,000. Only for Llama3.3 70B did we implement a lower output token limit of 5,000, since the model otherwise gets caught in endless repetitions on some queries.

\begin{table}[h]
  \centering
  \footnotesize
  \caption{Details on API providers, LLM identifiers, and costs.}
  \label{tab:llm-pricing}
  \begin{tabular}{@{}llcc@{}}
    \toprule
    \textbf{API Provider} &
      \textbf{Model Identifier} &
      \textbf{Input \$ / M tok} &
      \textbf{Output \$ / M tok} \\
    \midrule
    \multicolumn{4}{l}{\emph{Models accessed through the OpenAI API}} \\
    OpenAI & \texttt{gpt-4.1-2025-04-14} & 2.00 & 8.00 \\
    OpenAI & \texttt{o3-2025-04-16}       & 2.00 & 8.00 \\
    \midrule
    \multicolumn{4}{l}{\emph{Models accessed through the Fireworks API}} \\
    Fireworks & \texttt{deepseek-r1-0528}           & 3.00 & 8.00 \\
    Fireworks & \texttt{qwen3-235b-a22b}            & 0.22 & 0.88 \\
    Fireworks & \texttt{llama-v3p3-70b-instruct}    & 0.90 & 0.90 \\
    Fireworks & \texttt{llama-v3p1-405b-instruct}   & 3.00 & 3.00 \\
    \bottomrule
  \end{tabular}
\end{table}

We note that we measured negligible roundtrip latency (less than 300ms) to both API endpoints.

\section*{Appendix D - Cascading Setup}

Similar to \citealp{zellinger2025}, we use self-verification (also known as \textit{P(True)}, from \citealp{kadavath2022}) to estimate an LLM's confidence to correctly answer a query. Specifically, given a query, the LLM sends itself a follow-up verification prompt asking whether the proposed answer is correct. Since the response to this query is a single token (Yes/No), we extract the estimated probability of correctness $p$ directly from the LLM's auto-regressive next-token probability. This $p$ is the \textit{self-verification correctness probability}.

To select the optimal confidence threshold for a use case $\lambda$, we maximize the cascade's expected reward $R(\lambda; \theta)$ for $\theta$ ranging over all 2.5\% quantiles of empirically observed self-verification correctness probabilities on the training set. To evaluate the cascade's performance, we fix the optimal confidence thresholds $\theta^{*} = \theta^{*}(\lambda)$ (dependent on $\lambda$) and compute the expected rewards $R(\lambda; \theta^{*}(\lambda))$ on the test set.

\section*{Appendix E - Prompt Templates}

\lstset{
  basicstyle   = \ttfamily\small,
  escapeinside = {(*@}{@*)}   
}

This appendix reproduces verbatim the prompt templates used in our
experiments. Placeholders are printed in \textbf{bold} and are wrapped in curly braces.

\subsection{MATH Benchmark — Problem-Solving Prompts}
\label{app:math-prompts}

\noindent \textbf{System prompt:} \begin{lstlisting}
Your task is to solve a math problem. First think step-by-step,
then end by giving your final answer in the form
'Final Answer: x', where x is the final answer.
DO NOT say anything after that. Make sure to end on the numeric answer.
\end{lstlisting}

\noindent \textbf{User prompt:} 
\begin{lstlisting}
Your task is to solve the following math problem: (*@\textbf{\{problem\}}@*)

Reason step-by-step, then give your final by saying
'Final Answer: x', where x is the final numeric answer.
DO NOT say anything after that. Make sure to end on the numeric answer.
\end{lstlisting}

\subsection{MATH Benchmark — Evaluation Prompts}
\label{app:self-verification-prompts}

\noindent \textbf{System prompt:} 
\begin{lstlisting}
Your task is to determine if an AI model's solution to a
college-level math problem is correct.
If the solution is correct, output "Y".
Otherwise, output "N".
Only output "Y" or "N", nothing else.
\end{lstlisting}

\noindent \textbf{User prompt:} 
\begin{lstlisting}
Consider a proposed solution to the following math problem:

Problem:
(*@\textbf{\{problem\}}@*)

Proposed solution:
(*@\textbf{\{proposed\_sol\}}@*)

Decide if the proposed solution is correct.
Only output "Y" or "N", nothing else.

Correct?
\end{lstlisting}

\end{document}